\documentclass[11pt]{article}
\usepackage{enumerate}
\usepackage{pdfsync}
\usepackage[OT1]{fontenc}
\usepackage{wrapfig}   %

\usepackage{algorithm}
\usepackage{algpseudocode}  %
\usepackage{amsmath} 
\usepackage{url}           
\usepackage{booktabs}    
\usepackage{amsfonts}      
\usepackage{nicefrac}      
\usepackage{microtype}    
\usepackage{color}
\usepackage[colorlinks,
            linkcolor=red,
            anchorcolor=blue,
            citecolor=blue, pagebackref=true
           ]{hyperref}
\usepackage{fullpage}
\usepackage{setspace}
\usepackage{tabularx}
\usepackage{float}
\usepackage{wrapfig,lipsum}
\usepackage{enumitem}
\usepackage{natbib}
\usepackage{multirow} 
\usepackage[dvipsnames]{xcolor}
\usepackage{subcaption}
\usepackage{algorithm}
\usepackage{algpseudocode}

\makeatletter
\newcommand*{\rom}[1]{\expandafter\@slowromancap\romannumeral #1@}
\makeatother
\usepackage{amsmath}

\usepackage{mmll}
\usepackage{fancyvrb}
\usepackage{cleveref}
\usepackage{pifont}%

\usepackage{listings}
\usepackage{xcolor, colortbl}

\usepackage[normalem]{ulem}
\usepackage{tcolorbox}

\let\oldnl\nl
\newcommand{\nonl}{\renewcommand{\nl}{\let\nl\oldnl}}%

\usepackage{algpseudocode}

\newcommand{\nop}[1]{}

\algnewcommand{\LeftComment}[1]{\Statex \hspace{-\algorithmicindent} \(\triangleright\) #1}
\usepackage{amsthm}

\definecolor{softpurple}{RGB}{150, 120, 180}

\definecolor{myviolet}{RGB}{138, 43, 226}  %

\allowdisplaybreaks

\begin{document}
\title{\huge Group-Sensitive Offline Contextual Bandits}
\author{
   Yihong Guo\thanks{ 
   Johns Hopkins University; email: {\tt
   yguo80@jhu.edu
   }} 
   ~~~~~~
   Junjie Luo\thanks{ 
   Johns Hopkins University; email: {\tt
   jluo41@jhmi.edu
   }} \\
   Guodong (Gordon) Gao\thanks{ 
   Johns Hopkins University; email: {\tt  
    gordon.gao@jhu.edu
   }} 
   ~~~~~~
   Ritu Agarwal\thanks{ 
   Johns Hopkins University; email: {\tt  
   ritu.agarwal@jhu.edu
   }} 
   ~~~~~~
   Anqi Liu\thanks{
   Johns Hopkins University; email: {\tt
   aliu.cs@jhu.edu}
   }   
}
\date{}
\maketitle

\begin{abstract}
Offline contextual bandits allow one to learn policies from historical/offline data without requiring online interaction.
However, offline policy optimization that maximizes overall expected rewards can unintentionally amplify the reward disparities across groups. 
As a result, some groups might benefit more than others from the learned policy, raising concerns about fairness, especially when the resources are limited. 
In this paper, we study a group-sensitive fairness constraint in offline contextual bandits, reducing group-wise reward disparities that may arise during policy learning. 
We tackle the following common-parity requirements: the reward disparity is constrained within some user-defined threshold or the reward disparity should be minimized during policy optimization. 
We propose a constrained offline policy optimization framework by introducing group-wise reward disparity constraints into an off-policy gradient-based optimization procedure. 
To improve the estimation of the group-wise reward disparity during training, we employ a doubly robust estimator and further provide a convergence guarantee for policy optimization.
Empirical results in synthetic and real-world datasets demonstrate that our method effectively reduces reward disparities while maintaining competitive overall performance.
\end{abstract}

\section{Introduction}
Contextual bandits \citep{bouneffouf2020survey} are a class of sequential decision-making problems in which, at each round, an agent observes a context and pulls an arm from the action set, then receives the corresponding reward. 
The goal is to maximize the cumulative reward. 
Unlike full-feedback settings, the agent does not observe rewards for unseen actions. 
In many real-world settings (e.g., news recommendation, medical treatment), collecting new interaction data to learn a policy is often costly, risky, or infeasible.
Instead, one often only has access to historical interactions and learn a new policy from the offline logging dataset. 
This offline contextual bandit setting is essential when safety constraints or budget limitations prohibit further experimentation, yet one still wishes to leverage logging data for policy evaluation and improvement.

In many applications, reward represents the benefit an individual receives from the environment (e.g., access to resources, service quality, health). Offline policy optimization seeks to maximize expected reward but can also create or amplify group reward disparities in gains/outcomes. Consider sending prescription pickup reminders task formulated as an offline contextual bandit: context = patient features, action = when to send, reward = whether pickup and the offline data is the historical data, consisting of (patient, message sending time, whether pick up). Picking up the prescription will potentially benefit the patient's health.
The offline dataset shows that female patients tend to have higher pickup rates than male patients due to unobserved behavioral or contextual differences, as shown in the left two columns of \Cref{fig:smsdata_intro}. When we optimize a policy to maximize the expected pickup rate, we may learn to reinforce message styles or timing that align more with what works for female patients to easily gain rewards, leading to a higher reward disparity as illustrated in \Cref{fig:smsdata_intro}. Such improvement leads to a more unfair policy. 
In high-stakes domains (healthcare, education, public resources) with limited resources, such disparities are especially problematic: unfair policies that boost one group’s outcomes can disproportionately harm others. This raises urgent concerns about fairness, accountability, and safe, equitable decision-making in offline learning.

\begin{wrapfigure}{r}{0.4\textwidth}
    \centering
\includegraphics[width=\linewidth]{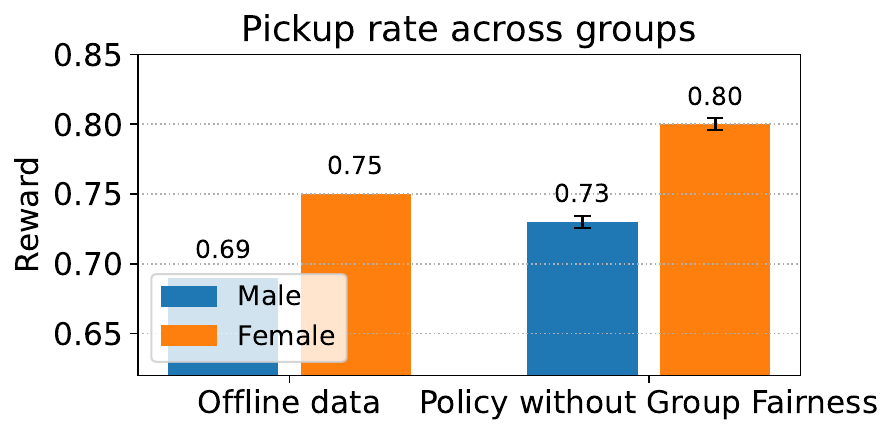}
    \caption{The reward of two groups from a real pick-up reminder message sending task where the goal is to improve the pickup rate of prescriptions. Policy optimization without group-sensitive fairness increases the reward disparity.
    }
    \label{fig:smsdata_intro}
\end{wrapfigure}

To summarize, in the offline contextual bandit setting, we identify an illustrative scenario of group-level reward disparity through real-world data and experiments. Specifically, optimizing a policy solely for reward can introduce or amplify reward disparities, resulting in uneven benefit distribution and raising concerns about fairness. Prior work on fairness in bandit settings has primarily addressed fairness in action exposure \citep{patil2021achieving, wang2021fairness} or fairness, such as individual fairness \citep{huang2021fairness,huang2022achieving}, rather than accounting for disparities in the rewards received by different demographic groups.
RobinHood \citep{metevier2019offline} proposes a meta-algorithm for the offline contextual bandit with fairness guarantees on a broad range of fairness metrics, that penalize the reward with the upper bound of the fairness statistical test. However, the optimization process is a black-box optimization algorithm, making it hard to apply to large-scale and high-dimensional datasets. 
Despite having a high probability guarantee, the algorithm may lead to overly conservative results and a severe trade-off between fairness and reward.

In this paper, we propose to mitigate the reward disparity issue through constrained offline policy optimization so that 
the reward disparity is smaller than a user-defined threshold. Specifically, we propose the off-policy group-constrained policy gradient method (GC-PG) and account for the constraint with the Lagrangian method. We employ a doubly robust estimator (DR) to obtain an estimator that exhibits low bias and variance for policy gradient and constraint estimation. Our contributions are as follows:
\begin{itemize}[itemsep=1pt,topsep=0pt,parsep=0pt,leftmargin=*]
    \item We tackle the group-sensitive reward disparity amplification in offline contextual bandit settings. We develop an off-policy group-constrained policy gradient algorithm (GC-PG), allowing for flexible control over group disparities through user-defined thresholds.
    \item We provide a theoretical guarantee of convergence, showing that our algorithm converges to a stationary point at a rate of $O(1/T)$.
    \item We validate our approach on synthetic offline contextual bandit problems converted from real-world classification datasets and a real-world prescription reminder message timing problem, demonstrating that our method effectively mitigates reward disparities without compromising overall policy performance. We also discuss the potential trade-off between fairness and performance, and an alternative solution to better balance the reward and fairness. 
\end{itemize}

\section{Related Work}

\textbf{Fairness in machine learning}
Fairness in machine learning has been extensively studied across supervised learning settings \citep{pessach2022review}. Classical notions of fairness include statistical parity, equalized odds, and individual fairness, with algorithms proposed to mitigate group disparities via pre-processing \citep{pessach2022review,
calmon2017optimized}, in-processing \citep{calders2010three,samadi2018price}, or post-processing techniques \citep{corbett2017algorithmic,
menon2018cost}. Recent advances have extended fairness considerations to sequential decision-making, such as bandits. One line of fairness bandit works considers the ``fair exposure'' of the action to avoid unfair winner-takes-all allocation of exposure \citep{patil2021achieving, wang2021fairness}. 
These works did not consider sensitive features, but focused more on the fairness of the action selection process. \cite{schumann2019group} consider a setting where the arms are partitioned into different sensitive groups, and the chance of pulling those arms should be equal. \cite{huang2021fairness, huang2022achieving} focus on individual fairness and propose that, in online recommendations, similar people should receive similar rewards despite the sensitive features. Similar to our problem, RobinHood \citep{metevier2019offline} proposed a meta-algorithm for fairness bandit with user-defined fairness metrics. However, the black box optimization method and the high probability guarantee to satisfy make it hard to scale to large and high-dimensional datasets, and might lead to severe performance and fairness trade-offs. In contrast, we propose a more practical policy gradient method that is more scalable to large datasets, through which we also achieve more favorable tradeoffs.

\textbf{Offline contextual bandit}
Many methods have been proposed to optimize off-policy in offline contextual bandits. One can first estimate the missing bandit feedback with a direct method \citep{zadrozny2003cost,
beygelzimer2009offset,guo2024distributionally}, propensity scores \citep{bottou2013counterfactual,
swaminathan2015counterfactual}, or doubly robust estimator \citep{dudik2011doubly,
zhao2015doubly} from offline data. Recently, many other robust/pessimistic methods have been proposed to conservatively estimate the missing feedback \citep{li2022pessimism,
yang2023distributionally}. With such counterfactual reward estimation, the policy gradient \citep{sutton1999policy} is a popular method for optimizing the parameterized policy in offline bandit and RL, and typically requires on-policy sampling with the policy. In our paper, we utilize the doubly robust estimator for counterfactual reward estimation and the evaluation of the fairness constraint.

\section{Method}
\subsection{Problem Formulation}

In the contextual bandit, at each round $t$, the agent observes a context vector $x_t \in \mathcal{X}$, selects an action $a_t \in \mathcal{A}$, and receives a corresponding reward $r(x_t, a_t)$. 
A policy $\pi(a|x)$ maps a context $x$ to a distribution over the action space. 
The objective of policy optimization is to find a policy $\pi$ that maximizes the expected reward, also known as the \emph{value function}, defined as
    $V(\pi) = \mathbb{E}_{x \sim \mathcal{P}_x,\, a \sim \pi(\cdot \mid x)} \left[ r(x, a) \right]$,
where $\mathcal{P}_x$ denotes the (unknown) distribution over contexts.

In the offline contextual bandit, the agent has access to a fixed offline dataset $\mathcal{D} = \{(x_i, a_i, r_i)\}_{i=1}^n$, collected by a logging policy $\pi_\beta$. 
The \emph{offline contextual bandit} problem learns a new policy $\pi$ from $\mathcal{D}$ that achieves high value $V(\pi)$ under the true environment dynamics.

\paragraph{Group-Sensitive Constraint} 
We consider fairness across a sensitive attribute $S \in \{0,1\}$ (e.g., gender, race, or age group) that is part of the context. Note that we consider the two-group case. We denote the two groups as
    $X^0 = \{x \in X \mid S = 0\}, \quad X^1 = \{x \in X \mid S = 1\}$, 
and the $P_{X^0}$ and $P_{X^1}$ are the distributions of the group  $X^0$ and $X^1$.
We define the \emph{reward disparity} $F(\pi)$ between the two groups under a policy $\pi$ as
\[
    F(\pi) = | \mathbb{E}_{x \sim \mathcal{P}_{X^0},\, a \sim \pi(\cdot \mid x)} [ r(x, a) ] - \mathbb{E}_{x \sim \mathcal{P}_{X^1},\, a \sim \pi(\cdot \mid x)} [ r(x, a) ] |.
\]

A \emph{group-sensitive constraint} requires that the reward disparity is bounded:
    $F(\pi) \leq \epsilon$,
where $\epsilon > 0$ is a tolerance parameter that can be flexibly chosen based on application-specific requirements or based on the observed reward disparity under the logging policy $\pi_\beta$, i.e., $ \epsilon = F(\pi_\beta)$. Thus, the objective of offline contextual bandit problem under group-sensitive fairness constraints can be formalized as:
\[
    \max_{\pi} \quad  V(\pi), \quad
    \text{subject to} \quad F(\pi) \leq \epsilon.
\]

\subsection{Off-Policy Group-Constrained Policy Gradient}

In this section, we propose the Off-Policy Group-Constrained Policy Gradient algorithm (GC-PG) to learn a policy from offline data that maximizes the expected reward while improving fairness across different demographic groups. Specifically, we propose a policy gradient approach to optimize the policy and employ a Lagrangian method to enforce constraint satisfaction, ensuring that the reward disparity across groups remains bounded or is reduced. The learning process involves updating both the policy parameters $\theta$ and the dual variables $(\lambda, \eta)$ iteratively. We first present how we perform policy gradient updates; then, we introduce the Lagrangian relaxation for the constraint. At last, we propose to use the doubly robust estimator to mitigate the bias and variance when evaluating group-sensitive fairness satisfaction for the Lagrangian method as well as the policy gradient.

\textbf{Policy Gradient.} We parameterize the policy $\pi$ with parameters $\theta$, denoted as $\pi_\theta(a \mid x)$, such that $\int_{a} \pi_\theta(a|x) da = 1$ for all $ x \in \mathcal{X}$. In offline contextual bandit, the policy seeks to maximize the expected reward:
   $ \max_{\pi_\theta} \mathbb{E}_{x \sim \mathcal{P}_x,\, a \sim \pi(\cdot \mid x)} \left[ r(x, a) \right]$. 
 Then, the gradient of the value function is given by:
\begin{align}
   \label{eq: vanilla policy gradient}
    \nabla_{\theta}V(\pi_{\theta}) = \mathbb{E}_{x \sim \mathcal{P}_x,\, a \sim \pi(\cdot \mid x)} \left[ \nabla \log \pi_{\theta} (a|x)r(x, a) \right]. 
\end{align}
In the offline contextual bandit setting, given a context $x$, we might not have access to the $r(x, \pi_{\theta}(x))$ but only have the reward of a context action pair from an unknown logging policy. Thus, we estimate the reward $\hat{r}(x, a)$ from the static offline dataset. And we can obtain the policy gradient through simulation of the context action pair $\mathcal{D}_{\pi_\theta} = \{(x_i, \pi_\theta(a|x)) \}$, and the approximation policy gradient is:
\[
    \textstyle
    \hat{g}(\theta) =  \frac{1}{|\mathcal{D}_{\pi_\theta}|} \sum_{i \in \mathcal{D}_{\pi_\theta}}\nabla \log \pi_{\theta} (a_i|x_i) \hat{r}(x_i, a_i),
\]
where $\hat{r}$ is the reward estimator. In the paper, we use the XGBoost model to estimate the reward. 

\textbf{Lagrangian Relaxation.}
To handle the group-sensitive constraint, we introduce a Lagrangian formulation that allows soft constraint satisfaction. Note that here, we only consider two groups; however, our method can be extended to multiple group settings, which will be discussed later. Specifically, the Lagrangian objective is given by:
\begin{align*}
\textstyle
\min_{\lambda \geq 0} \max_\theta V(\pi_\theta) + \lambda & \big[\epsilon - | \mathbb{E}_{x \sim \mathcal{P}_{\mathcal{X}^0},\, a \sim \pi(\cdot \mid x)} r(x, a)  - \mathbb{E}_{x \sim \mathcal{P}_{\mathcal{X}^1},\, a \sim \pi(\cdot \mid x)}  r(x, a) | \big]
\end{align*}
Removing the absolute value and rewriting this:
\begin{align}
\label{eq:rewrite lagrangian objective}
& \min_{\lambda \geq 0, \eta > 0}  \max_\theta V(\pi_\theta) \\ &+ \lambda \bigg[\epsilon - \big[\mathbb{E}_{x \sim \mathcal{P}_{\mathcal{X}^0},\, a \sim \pi_\theta(\cdot \mid x)} r(x, a) - \mathbb{E}_{x \sim \mathcal{P}_{\mathcal{X}^1},\, a \sim \pi_\theta(\cdot \mid x)}  r(x, a)  \big]\bigg]\notag\\
& + \eta \bigg[\epsilon - \big[ \mathbb{E}_{x \sim \mathcal{P}_{\mathcal{X}^1},\, a \sim \pi_\theta(\cdot \mid x)}  r(x, a) - \mathbb{E}_{x \sim \mathcal{P}_{\mathcal{X}^0},\, a \sim \pi_\theta(\cdot \mid x)}  r(x, a)  \big] \bigg], \notag
\end{align}
where $\lambda, \eta \geq 0$ are dual variables that adaptively penalize group-sensitive constraint violations in either direction. Optimizing this Lagrangian enables us to maximize the reward while minimizing fairness violations through the dynamic adjustment of $\lambda$ and $\eta$.

\textbf{Policy Learning.}
At each iteration, we update the policy parameters $\theta$ via gradient ascent to maximize the Lagrangian $L$.
The policy gradient is derived as:
\begin{align*}
    \nabla_\theta L 
    &= \mathbb{E}_{x,a \sim \pi_\theta} \nabla_\theta \log \pi_\theta(a \mid x) r(x,a)  \\
    & \quad - (\lambda - \eta) \, \mathbb{E}_{x,a \sim \pi_\theta, X^1} [\nabla_\theta \log \pi_\theta(a \mid x) r(x,a) ] \\
    & \quad + (\lambda - \eta) \, \mathbb{E}_{x,a \sim \pi_\theta, X^0} [\nabla_\theta \log \pi_\theta(a \mid x) r(x,a) ]\\
    & = \mathbb{E}_{x,a \sim \pi_\theta, X} \nabla_\theta \log \pi_\theta(a \mid x) r(x,a)  \big[ I(x \in X^1) (1 - \lambda + \eta) +I(x \in X^0) (1 + \lambda - \eta) \big].
\end{align*}
This formulation suggests that the algorithm will ``rescale'' the gradient of different sensitive groups. Or from another perspective, the algorithm ``up-weights'' the reward of the lower reward group and ``down-weights'' the higher reward group. 

In practice, these expectations are estimated using mini-batches of the simulated offline data $\mathcal{S}_{\pi_{\theta}} = \{(x_i, \pi_{\theta}(x_i), \hat{r}(x_i, \pi_{\theta}(x_i)))\}$. Then, the empirical policy gradient is approximated as:
\begin{align}
    \label{eq:policy_gradient_with_lambda}
    \hat{g}(\hat{\theta})
    &=  \frac{1}{n} \sum_{i=1}^n \nabla_\theta \log \pi_\theta(a_i \mid x_i) \hat{r}(x_i, a_i)[I(x_i \in X^1)(1 - \lambda + \eta) +I(x_i \in X^0) (1 + \lambda - \eta)].  
\end{align}

The $\theta$ are then updated via gradient ascent:
    $\theta = \theta + \alpha \hat{g}(\theta)$,
where $\alpha > 0$ is the policy learning rate.

\textbf{Dual Variable Update.}
At each optimization step, we update the dual variables $\lambda$ and $\eta$ via gradient descent to enforce the group-sensitive constraint. By taking the derivative of the $\lambda$ and $\eta$, we have the gradient for the dual variables: 
\begin{align*}
    \nabla_\lambda L &= \epsilon - \big[ \mathbb{E}_{x,a \sim \pi_\theta, X^1} r(x,a) - \mathbb{E}_{x,a \sim \pi_\theta, X^0} r(x,a) \big], \\
    \nabla_\eta L &= \epsilon - \big[ \mathbb{E}_{x,a \sim \pi_\theta, X^0} r(x,a) - \mathbb{E}_{x,a \sim \pi_\theta, X^1} r(x,a) \big].
\end{align*}
In this dual variable update, we need to perform the off-policy evaluation to approximate the reward in two groups: $\hat{V}_{X^k}(\pi_\theta)$, $k \in \{0,1\}$:
\begin{align}
\textstyle
    \label{eq: ope group reward}
    &\hat{V}_{X^k}(\pi_\theta) = \frac{1}{\sum_i^n I(x_i \in X^k)}  \sum_{i}^{n} \hat{r}(x_i,\pi_\theta(x_i)) I(x_i \in X^k)
\end{align}

Then the approximated gradient of $\lambda$ and $\eta$ are:
\begin{align*}
    & \hat{g}(\lambda) = \epsilon - \big[ \hat{V}_{X^1}(\pi_\theta) - \hat{V}_{X^0}(\pi_\theta)\big],  \\
    & \hat{g}(\eta)  = \epsilon - \big[ \hat{V}_{X^0}(\pi_\theta) - \hat{V}_{X^1}(\pi_\theta)\big].
\end{align*}

We then perform the updates:
    $\lambda \leftarrow \lambda - \beta \hat{g}(\lambda),\eta \leftarrow \eta - \beta \hat{g}(\eta)$,
and $\beta > 0$ is the dual learning rate.

\textbf{Doubly robust estimator.} However, using vanilla reward estimation for policy optimization and off-policy evaluation is highly biased, especially when the logging policy is very different from the learning policy $\pi_{\theta}$ \citep{dudik2011doubly}. Such high bias will lead to a biased estimation of the policy gradient in \Cref{eq: vanilla policy gradient} and also the reward disparity in \Cref{eq: ope group reward}. which directly affects constraint estimation. Therefore, we propose using a doubly robust (DR) estimator to minimize bias while maintaining low variance. The DR estimator combines an importance-weighted value estimate with a model-based estimate of the reward, offering robustness to misspecification in either component. Specifically, the DR estimators are given by:
\begin{align}
    \hat{V}^{\text{DR}}(\pi_\theta) &= \frac{1}{n} \sum_{i=1}^n \bigg[ \hat{r}(x_i, \pi_{\theta}(x_i))  + \frac{\pi_{\theta}(a_i \mid x_i)}{\pi_b(a_i \mid x_i)} \big[ r_i - \hat{r}(x_i, a_i)\big ] \bigg], \notag \\
    \label{eq: dr x0}
    \hat{V}^{\text{DR}}_{X^k}(\pi_\theta) &= \frac{1}{\sum_i^n I(x_i \in X^k)} \sum_{i=1}^n \bigg[ \hat{r}(x_i, \pi_{\theta}(x_i)) + \frac{\pi_{\theta}(a_i \mid x_i)}{\pi_b(a_i \mid x_i)} \big[ r_i - \hat{r}(x_i, a_i) \big] \bigg ]I(x_i \in X^0), k \in \{0,1\},
\end{align}
where $\pi_b$ is the logging policy, $\pi_{\theta}$ is the target policy, $\hat{r}$ is an estimate of the action-value function, and $(x_i, a_i, r_i)$ are samples collected under $\pi_b$. This estimator corrects for bias in model-based predictions using importance-weighted observed rewards, and provide a more accurate and stable estimate of the policy value and gradient. For convenient, we use the notation $\hat{r}_{\text{dr}}(x_i,\pi_\theta(x_i))$ to replace the DR estimator.
Then the gradient of the policy and the dual variable are:
\begin{align}
    \label{eq: policy optimization}
    &\hat{g}(\hat{\theta}) = \frac{1}{n} \sum_{i=1}^n \nabla_\theta \log \pi_\theta(a_i \mid x_i) \hat{r}_\textbf{dr}(x_i, a_i)  \big[I(x_i \in X^1)(1 - \lambda + \eta) +I(x_i \in X^0) (1 + \lambda - \eta) \big],\\
    &\hat{g}(\lambda)  = \epsilon - \big[  \hat{V}^{\text{DR}}_{X^1}(\pi_\theta) - \hat{V}^{\text{DR}}_{X^0}(\pi_\theta)\big], \nonumber\\
    &\hat{g}(\eta) = \epsilon - \big[\hat{V}^{\text{DR}}_{X^0}(\pi_\theta) - \hat{V}^{\text{DR}}_{X^1}(\pi_\theta)  \big].\nonumber
\end{align}

\paragraph{Overall Procedure.}
Our method alternates between updating the policy parameters to maximize reward and adjusting the dual variables to enforce fairness. 
When the reward disparity between $X^0$ and $X^1$ exceeds $\epsilon$, the corresponding dual variable increases to impose a stronger penalty, guiding the policy back toward fairness. 
When the fairness constraint is satisfied, the dual variables decrease, allowing the policy to prioritize reward optimization. We summarize the it in \Cref{algo: main}

\textbf{Extension to Multiple Groups} Naively extending the method to multiple groups induces pair-wise constraints in \Cref{eq:rewrite lagrangian objective}, which is $O(G^2)$ constraints per update. Such formulation is not practical when there are more than three groups. Instead of considering pair-wise constraints every step, we can optimize a max-pair surrogate: at each step, we identify the most violated group pair $(i,j)$, i.e., $\argmax_{i,j} | R_i - R_j| $, and do the constraint optimization only on that group pair. We refer to the detailed formulation, discussion, and the experimental results in the \Cref{sec: exp}. 

\begin{algorithm}[t]
\caption{Group-sensitive offline contextual bandit}
\label{algo: main}
\begin{algorithmic}[1]
    \State \textbf{Input:} Offline dataset $S = \{x, a, r, \pi_\beta(a \mid x)\}$, threshold $\epsilon$, sensitive feature
    \State \textbf{Output:} Learned policy $\pi_\theta$
    \State \textbf{Init:} Policy $\pi_\theta$, reward model $\hat{r}$, policy learning rate $\alpha$, dual update rate $\beta$
    \State Learn reward estimator $\hat{r}$ from $S$
    \For{$t = 0, \dots, T$}
        \State Sample context–action pairs $D_{\pi_\theta}$
        \State Update $\theta_{t+1} = \theta_t + \alpha \hat{g}(\theta_t)$ using \Cref{eq:policy_gradient_with_lambda}
        \State Compute $\hat{V}_{X^0}^{\mathrm{DR}}(\pi_{\theta_{t+1}})$ and $\hat{V}_{X^1}^{\mathrm{DR}}(\pi_{\theta_{t+1}})$ via \Cref{eq: dr x0} and update the dual variable
    \EndFor
    \State \textbf{return} $\pi_{\theta_T}$
\end{algorithmic}
\end{algorithm}

\section{Theoretical Analysis}
In this section,  we provide the convergence analysis of our algorithm. We defer the proof to \Cref{section: proof}.
\begin{assumption}
    Let the $\pi_\theta(a|x)$ be a policy. There exists a $G$ such that for any context $x$, the log-density of the policy function is bounded by $G$, i.e. $\| \nabla_\theta \log \pi_\theta(a|x) \|_2 \leq G$. 
\end{assumption}
This assumption is a general assumption that the policy changes smoothly, which is widely seen in previous contextual bandit literature \citep{allen2016variance,xu2020improved,yang2023distributionally}.
\begin{assumption}
    \label{assumption:dual bounded}
    Let $\lambda$ and $\eta$ be the dual variable during training, there exists an $B$ such that the  $\lambda$ and $\eta$ are upper bounded by $B$. 
\end{assumption}
This assumption can be easily satisfied by projecting the $\lambda$ and $\eta$ to the $(0, B]$ interval at each update. In the policy optimization objective \Cref{eq: policy optimization},  the reward is scaled by either $1-\lambda + \eta$ or $1+\lambda - \eta$, which can become negative. If this scaling factor turns negative, it means we are explicitly harming one group’s performance in order to achieve fairness, which may be undesirable in real-world applications.

\begin{lemma} (Bias of the doubly robust estimator). The bias of the doubly robust estimator under the Rademacher complexity is bounded with $1-2\delta$ probability: 
\label{lemma: bias}
    \begin{align*}
        \text{Bias}^2(\hat{V}^\text{Dr}(\pi_\theta)) = E_{x}[(\hat{r}(x,a) - r(x,a))(1-\frac{\pi_b(a|x)}{\hat{\pi}_b(a|x)})]^2 \leq O(\frac{\log(1/\delta)}{|D|})
    \end{align*}
where $\pi_b$ is the logging policy and $\hat{\pi}_b(a|x)$ is the estimated logging policy. 
\end{lemma}
The bias of the DR estimator reduces to 0 with the convergence rate $O(1/|D|)$, relying on the estimation of the reward and the logging policy. 

We now present our main Theorems. 
\begin{theorem}
\label{theorem: reward mse}
The MSE between the DR estimator and the true reward is bounded with probability $1-2\delta$ as follows:
\begin{align*}
        \text{MSE}_\text{DR} = & \mathbb{E}_{(x_i, a_i, r_i) \sim D} \bigg[ \hat{r}(x_i, \pi_{\theta}(x_i))  + \frac{\pi_{\theta}(a_i \mid x_i)}{\pi_b(a_i \mid x_i)} ( r_i - \hat{r}(x_i, a_i)) - r(x_i, \pi_\theta( x_i)) \bigg]^2 \nonumber  \\ 
        = & \text{Var}[\hat{V}^{\text{DR}}(\pi_\theta)] + \text{Bias}^2(\hat{V}^{\text{DR}}(\pi_\theta)) \\
        \leq & O(1/|D|) + O(\frac{\log(1/\delta)}{|D|})
    \end{align*}
\end{theorem}

\textbf{Remark.} \Cref{theorem: reward mse} shows that the mean squared error (MSE) of the DR estimator can be decomposed into the bias and variance terms. The first term is the variance of the doubly robust estimator, with a convergence rate of $O(1/|D|)$ as shown in \Cref{lemma: variance} using Theorem 2 in \cite{dudik2011doubly}. The second term is the $O(1/|D|)$ convergence rate of the square of the bias term. Intuitively, the MSE decreases to 0 with more training data. Therefore, when the reward model has low bias and low variance, the resulting MSE of the doubly robust estimator will be significantly improved.

\begin{theorem}
    \label{theorem: approximation error}
    (Estimation error of the gradient) 
    With probability at least $1-2\delta$, the generalization
bound of policy gradient estimation satisfies:  
\begin{align*}
\mathbb{E}_{\pi_{\theta_t}}\left[\left\| \hat{g}_{\theta_t} - \nabla V^\pi_{\theta_t} \right\|_2^2 \right] \leq G^2 (1 + 2B)^2 \text{MSE}_\text{DR} .
\end{align*}

\end{theorem}

\textbf{Remark.} The estimation error of the policy gradient consists primarily of the MSE in \Cref{theorem: reward mse}. Since the MSE decreases at a rate of $O(1/|D|)$, the generalization error diminishes as the size of the offline dataset increases, due to a better learned reward estimation model and logging policy estimation.

\begin{theorem}
    \label{theorem: gradient vanish}
    (Gradient vanish). Assuming the value function $V(\pi_\theta)$ is L-smooth for every $\theta$, let $T$ be the policy optimization steps, let the learning rate $\alpha = \frac{1}{4L}$. Then the learned policy $\pi_{\theta_T}$ satisfies:
    \begin{align*}
    \mathbb{E} [ \left\| \nabla V^\pi (\theta_{\hat{T}})\right\|_2^2 ]\leq&  \frac{4}{T\eta} ( V^\pi(\theta_{T+1}) - V^\pi(\theta_1))+ 3 G^2 (1+2B)^2  \text{MSE}_\text{DR} .
    \end{align*}
\end{theorem}

\textbf{Remark.} \Cref{theorem: gradient vanish} shows that the policy converges to a stationary point at a rate of $O(1/T)$, where $T$ is the number of optimization steps. The second term on the right-hand side captures the generalization error in the gradient estimation, as characterized in \Cref{theorem: approximation error}. This term decreases at a rate of $O(1/|D|)$, where $|D|$ is the number of offline data samples, showing the dependence of convergence accuracy on the dataset size.

\section{Experiments}
\label{sec: exp}
In this section, we present the experimental setup and results. We conduct experiments in 3 offline contextual bandit datasets generated from classification problems: Drug Consumption \citep{fehrman2017five}, Adult Income \citep{Dua2019}, and Student Performance Dataset \citep{cortez2008student}, where we validate the effectiveness of our method using different logging policies and sensitive features. We also apply our method to a real-world contextual bandit dataset for optimizing the timing of sending personalized prescription reminder messages. We show the data statistics in \Cref{{tab:data-stats}}.

\begin{table}[H]
\centering
\caption{Dataset Statistics and Sensitive Feature}
\label{tab:data-stats}
\resizebox{0.7\columnwidth}{!}{\begin{tabular}{lcccc}
\toprule
& \textbf{Drug} & \textbf{Adult} & \textbf{Student} & \textbf{Prescription} \\
\midrule
\textbf{Dimension}        & 28    & 50    & 13    & 1,578   \\
\textbf{\# of Arms}       & 4     & 2     & 4     & 12      \\
\textbf{\# of Samples}    & 1,876 & 48,759& 2,392 & 320,120 \\
\textbf{Feature}
& \shortstack{Gender\\Education} &\shortstack{Gender\\Race} & Education & Gender \\
\bottomrule
\end{tabular}}
\end{table}

\textbf{Baseline.}  We compare our method with RobinHood \cite{metevier2019offline}. RobinHood computes the high probability upper bounds $\hat{U}$ of the reward disparity and optimizes with the true reward if the upper bound satisfies the constraint and $-\hat{U}$ otherwise. We follow the original implementation and use CMA-ES \citep{hansen2001completely} to optimize the policy.

\textbf{Datasets.}
We introduce the classification datasets and their sensitive features below and explore sensitive features that may cause the reward disparity.
\begin{itemize}[itemsep=1pt,topsep=0pt,parsep=0pt,leftmargin=*]

    \item \textbf{Drug Consumption}: Predict drug use from demographic/psychometric features. We transform the original 7 classes into 4 bins (past month/ past year/ 5 years/ 10 years) and sensitive attributes are gender (male/female) and education (below/above college).
    
    \item \textbf{Adult Income}: Predict whether annual income exceeds \$$50K$ from demographic attributes. Sensitive attributes are gender (male/female) and race (White/Non-White).
    \item \textbf{Student Performance}: Predict final exam grade from demographic and school-related features (framed as classification). Sensitive attribute: parental education (below/above college).
\end{itemize}

\textbf{Logging policies} We test our algorithm on various offline contextual bandit scenarios by using different logging policy in the offline data $D_\text{offline} = \{x_i,a_i, r_i\}$. %
\begin{itemize}[itemsep=1pt,topsep=0pt,parsep=0pt,leftmargin=*]
\item \textbf{Random Policy}: A baseline policy that selects each action uniformly at random.

\item \textbf{Tweak-1 Policy}: A biased policy where one fixed action is chosen with probability $\rho$ (e.g., $\rho = 0.9$), and the remaining $1 - \rho$ probability is evenly distributed among other actions. In our experiments, we set $\rho = 0.9$, which is relatively a large shift.  

\item \textbf{Mixed Policy}: A learned logging policy obtained from the dataset, which reflects realistic but suboptimal decision-making, typically resulting in a relatively low expected reward.
\end{itemize}

\textbf{Synthetic offline bandit converted from classification.} We first convert a k-class classification task into a k-armed contextual bandit problem. In the classification task, we have $(x, y) \sim D$, where $x \in \mathcal{X}$ is the context vector and $y \in \mathcal{Y}$ is the label. We then generate the offline dataset from $D$ based on the different logging policy $\pi_\beta$, where the $\pi_\beta$ is defined above. Given a feature vector $x$, we choose a label (action) $a$ based on $\pi_\beta(a|x)$ and reveal the reward $r(x, a)$. The reward here is whether the selected action matches the ground truth label. The offline data will be in the form of $(x, a, r(x, a),\pi_\beta(a|x))$. During the training, we assume the logging policy is known. We use 70\% samples for training and 30\% for testing.

\textbf{Optimal timing for personalized prescription reminder messages}. This dataset is a real-world contextual bandit dataset that is collected in a study of the timing of prescription pick-up reminder messages. It contains historical records of patient information, message send times, and click responses. We define the context as the patient features, action as the message sending time, and reward as whether the patients pick up the prescription. The offline dataset contains (context, action, reward) information. In the offline dataset, messages are sent immediately after the medication becomes available, which serves as a proxy for a random logging policy. The goal is to learn a new policy from the offline dataset that selects optimal send times to increase the prescription pick-up rate and we consider \emph{gender} as the sensitive feature. 
\begin{table}[ht]
\centering
\caption{Comparison of original reward v.s. after policy optimization without constraint using a random logging policy. This demonstrates that the reward disparity increases with policy optimization without fairness. More results are referred to in the \Cref{app: additional exp}. }
\label{tab:before_after_diff_arrow}
\resizebox{0.7\columnwidth}{!}{\begin{tabular}{ccccccc}
\toprule
\multirow{2}{*}{\textbf{Data}}&\multirow{2}{*}{\textbf{Feature}} &\multicolumn{2}{c}{\textbf{Offline}} & \multicolumn{2}{c}{\textbf{Optimized}} & \multirow{2}{*}{\shortstack{\textbf{Disparity}\\ \textbf{Increase}}} \\
\cmidrule(lr){3-4} \cmidrule(lr){5-6}
&&\textbf{$X^0$} & \textbf{$X^1$} & \textbf{$X^0$} & \textbf{$X^1$} & \\
\midrule
 Adult & Race & $0.498$ & $0.501$ & $0.902$ & $0.737$ & $0.003 \rightarrow \mathbf{0.165}$ \\
 Adult & Gender & $0.500$ & $0.502$ & $0.768$ & $0.852$ & $0.002 \rightarrow \mathbf{0.084}$ \\
 Drug & Gender & $0.252$ & $0.231$ & $0.515$ & $0.426$ & $0.021 \rightarrow \mathbf{0.081}$ \\
 Drug & Edu & $0.250$ & $0.256$ & $0.425$ & $0.511$ & $0.006 \rightarrow \mathbf{0.086}$ \\
 Student & Edu & $0.239$ & $0.218$ & $0.694$ & $0.556$ & $0.021 \rightarrow \mathbf{0.138}$ \\
\bottomrule
\end{tabular}}
\end{table}

\textbf{Existence of reward disparity.}
In \Cref{tab:before_after_diff_arrow}, we show that the reward disparity increases with policy optimization and the unconstrained policies favor one group over another, showing performance gaps. 
Also, in \Cref{fig:smsdata}, we see that the policy optimization without constraint improves the engagement rate while also increasing the reward disparity. 
These disparities match the motivation for incorporating fairness constraints into offline policy optimization. 

\begin{table*}[t]
\centering
\setlength{\tabcolsep}{6pt}
\caption{Comparison of the \textbf{Unconstrained baseline} denoted as $\pi^\infty$ vs. \textbf{RobinHood} and \textbf{Ours}. Our method receives a lower reward disparity and also a similar overall reward compared with the unconstrained baseline. Our method also outperforms the RobinHood baseline, yielding higher rewards and a similar reward disparity around the $\epsilon$. More experiments on different $\epsilon$ are in \Cref{app: additional exp}.}
\label{tab:main}
\resizebox{\textwidth}{!}{\begin{tabular}{l l l | cc | cc | cc}
\toprule
\multicolumn{3}{c|}{} & \multicolumn{2}{c|}{\textbf{Unconstrained}} & \multicolumn{2}{c|}{\textbf{RobinHood}} & \multicolumn{2}{c}{\textbf{Ours}} \\
Dataset & Group & Off-policy & $\text{Reward}$ & $\text{Disparity}$ & $\text{Reward}$ & $\text{Disparity}$ & $\text{Reward}$ & $\text{Disparity}$ \\
\midrule
\multirow{6}{*}{Drug} & \multirow{3}{*}{\shortstack{Edu\\$\epsilon=0.03$}} 
  & Random  & $0.526_{\pm 0.034}$ & $0.086_{\pm 0.043}$ & $0.137_{\pm 0.009}$ & $0.024_{\pm 0.023}$ & $\mathbf{0.521}_{\pm 0.051}$ & $0.033_{\pm 0.025}$ \\
& & Tweak-1 & $0.501_{\pm 0.021}$ & $0.072_{\pm 0.031}$ & $0.200_{\pm 0.011}$ & $0.068_{\pm 0.082}$ & $\mathbf{0.492}_{\pm 0.029}$ & $0.009_{\pm 0.017}$ \\
& & Mixed   & $0.516_{\pm 0.051}$ & $0.073_{\pm 0.032}$ & $0.145_{\pm 0.017}$ & $0.013_{\pm 0.008}$ & $\mathbf{0.516}_{\pm 0.048}$ & $0.058_{\pm 0.021}$ \\
\cmidrule(lr){2-9}
& \multirow{3}{*}{\shortstack{Gender\\$\epsilon=0.03$}}
  & Random  & $0.470_{\pm 0.023}$ & $0.081_{\pm 0.015}$ & $0.123_{\pm 0.008}$ & $0.018_{\pm 0.003}$ & $\mathbf{0.489}_{\pm 0.035}$ & $0.046_{\pm 0.017}$ \\
& & Tweak-1 & $0.463_{\pm 0.050}$ & $0.094_{\pm 0.039}$ & $0.119_{\pm 0.003}$ & $0.008_{\pm 0.009}$ & $\mathbf{0.491}_{\pm 0.016}$ & $0.009_{\pm 0.005}$ \\
& & Mixed   & $0.476_{\pm 0.039}$ & $0.116_{\pm 0.041}$ & $0.173_{\pm 0.021}$ & $0.018_{\pm 0.024}$ & $\mathbf{0.509}_{\pm 0.042}$ & $0.012_{\pm 0.004}$ \\
\midrule
\multirow{6}{*}{Adult} & \multirow{3}{*}{\shortstack{Gender\\$\epsilon=0.07$}}
  & Random  & $0.792_{\pm 0.014}$ & $0.165_{\pm 0.015}$ & $0.757_{\pm 0.043}$ & $0.083_{\pm 0.011}$ & $\mathbf{0.801}_{\pm 0.038}$ & $0.061_{\pm 0.025}$ \\
& & Tweak-1 & $0.787_{\pm 0.039}$ & $0.172_{\pm 0.018}$ & $0.754_{\pm 0.042}$ & $0.046_{\pm 0.010}$ & $\mathbf{0.776}_{\pm 0.022}$ & $0.064_{\pm 0.019}$ \\
& & Mixed   & $0.797_{\pm 0.067}$ & $0.154_{\pm 0.019}$ & $0.761_{\pm 0.047}$ & $0.095_{\pm 0.011}$ & $\mathbf{0.801}_{\pm 0.087}$ & $0.063_{\pm 0.031}$ \\
\cmidrule(lr){2-9}
& \multirow{3}{*}{\shortstack{Race\\$\epsilon=0.05$}}
  & Random  & $0.805_{\pm 0.053}$ & $0.084_{\pm 0.005}$ & $0.775_{\pm 0.054}$ & $0.035_{\pm 0.004}$ & $\mathbf{0.798}_{\pm 0.052}$ & $0.058_{\pm 0.008}$ \\
& & Tweak-1 & $0.797_{\pm 0.047}$ & $0.092_{\pm 0.014}$ & $0.775_{\pm 0.032}$ & $0.051_{\pm 0.014}$ & $\mathbf{0.802}_{\pm 0.052}$ & $0.054_{\pm 0.017}$ \\
& & Mixed   & $0.805_{\pm 0.051}$ & $0.074_{\pm 0.027}$ & $0.790_{\pm 0.051}$ & $0.068_{\pm 0.019}$ & $\mathbf{0.817}_{\pm 0.054}$ & $0.055_{\pm 0.012}$ \\
\midrule
\multirow{3}{*}{Student} & \multirow{3}{*}{\shortstack{Edu\\$\epsilon=0.07$}}
  & Random  & $0.611_{\pm 0.017}$ & $0.138_{\pm 0.041}$ & $0.507_{\pm 0.038}$ & $0.042_{\pm 0.024}$ & $\mathbf{0.612}_{\pm 0.030}$ & $0.049_{\pm 0.024}$ \\
& & Tweak-1 & $0.672_{\pm 0.049}$ & $0.158_{\pm 0.039}$ & $0.490_{\pm 0.033}$ & $0.009_{\pm 0.005}$ & $\mathbf{0.624}_{\pm 0.029}$ & $0.035_{\pm 0.021}$ \\
& & Mixed   & $0.649_{\pm 0.042}$ & $0.132_{\pm 0.031}$ & $0.500_{\pm 0.041}$ & $0.044_{\pm 0.015}$ & $\mathbf{0.628}_{\pm 0.065}$ & $0.062_{\pm 0.014}$ \\
\bottomrule
\end{tabular}}
\end{table*}

\subsection{Experiment results}
 We evaluate the performance of our algorithm through the reward of different groups of people. We report the average reward and standard deviation among 30 runs. 
The quantitative results for the 3 synthetic cases are presented in \Cref{tab:main}. Each row in the tables corresponds to a specific configuration of sensitive attribute and fairness constraint level $\epsilon$, while the columns show the average rewards, as well as the resulting reward disparities under various logging policies.
For each dataset, we select values of $\epsilon$ based on the magnitude of reward disparities observed in the unconstrained setting, allowing us to meaningfully evaluate how fairness interventions affect policy behavior. Note that our setting is general enough to account for two types of reward parity: 1) reward disparity not widened after policy optimization and 2) resulting reward gap controlled within a tolerance level.  
\begin{figure}
    \centering
    \includegraphics[width=0.7\linewidth]{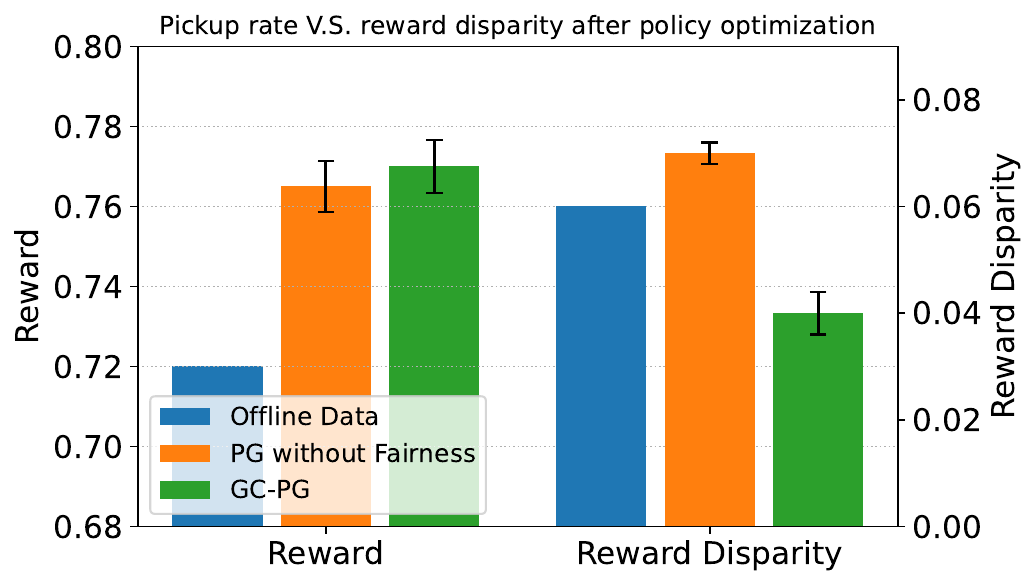}
    \caption{Results on prescription pickup reminder messages. GC-PG improves pickup rates and reduces reward disparity, achieving performance comparable to the unconstrained policy and lower disparity.}  
    \label{fig:smsdata}
\end{figure}

\textbf{Main results} Compared with unconstrained cases, our method can significantly reduce reward disparities given different fairness tolerance $\epsilon$ while maintaining similar rewards. 
Comparing our method with the RobinHood baseline, we can observe reward improvement in both groups, while maintaining a lower or similar reward disparity across all datasets. Especially with $\epsilon = 0.03$ in the Drug dataset, RobinHood tends to be overly conservative towards the fairness constraint when trying to achieve the high probability guarantee of fairness. Running the safety check with the upper bound of reward disparity at every iteration step may lead to very poor performance when $\epsilon$ is small.

\textbf{A closer look at prescription message timing} We present the results in \Cref{fig:smsdata}. In the left two columns, we observe that in the offline dataset, female patients exhibit a higher engagement rate ($0.73$) compared to male patients ($0.69$). The middle two columns show that optimizing the policy without fairness constraints leads to increased rewards for both groups $(0.80, 0.73)$ but also amplifies the reward disparity. In contrast, the right two columns demonstrate that our proposed algorithm, GC-PG, not only improves the rewards for both groups $(0.79,0.75)$ but also significantly reduces the disparity while maintaining overall performance comparable to the unconstrained case.

\begin{figure}
  \centering
  \resizebox{0.7\columnwidth}{!}{
  \begin{tabular}{cc} %
    \includegraphics[width=0.5\linewidth]{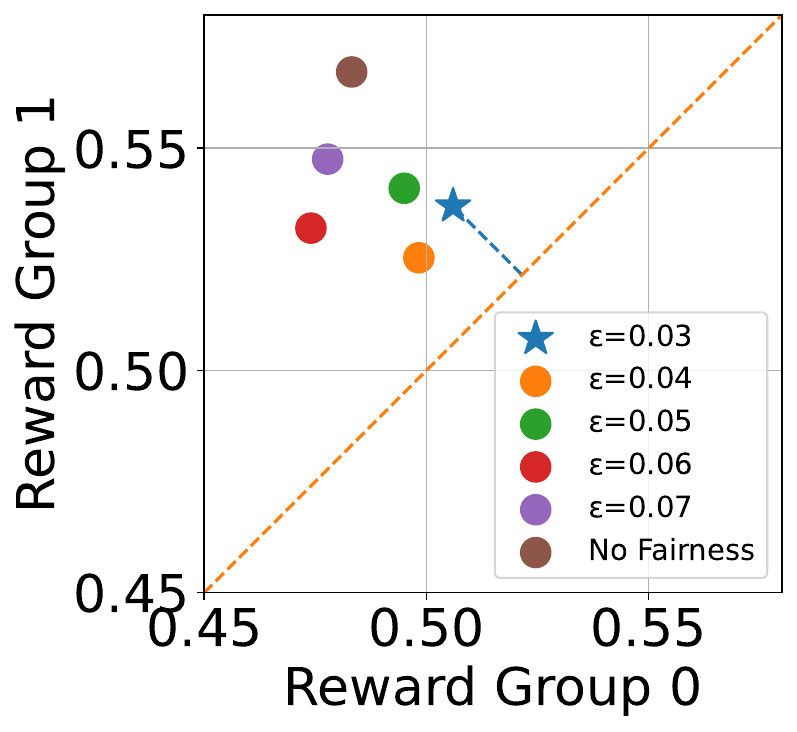} &
    \includegraphics[width=.5\linewidth]{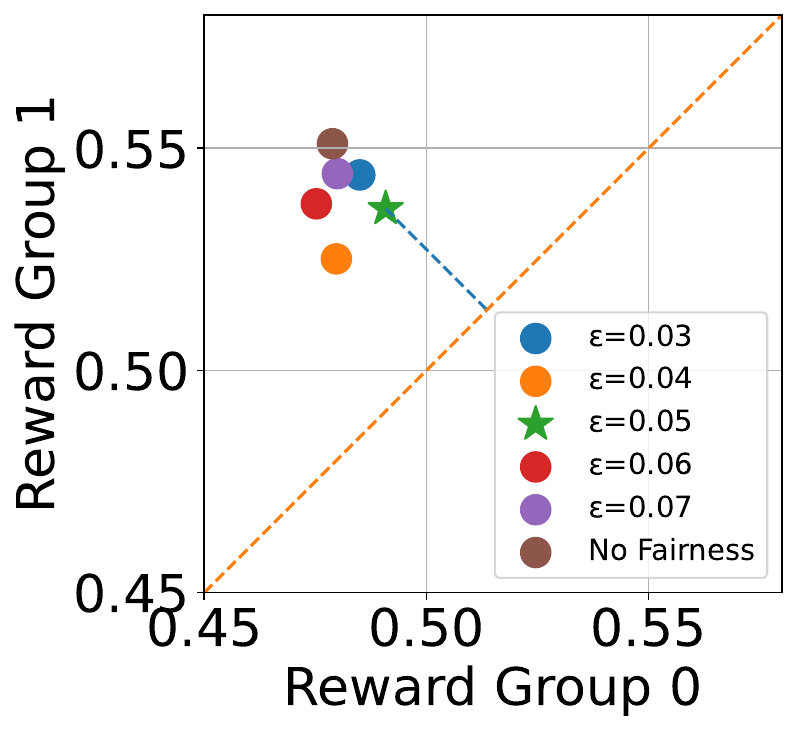} \\
    Random Policy &  Mixed Policy
  \end{tabular}}
  \caption{Scatter plot of reward $(R_1, R_2)$ on two groups using three different logging policies with different $\epsilon$ on drug dataset. Our method can reduce the reward disparity, but also has a trade-off between fairness and per-group reward. We mark the best policy, which is not Pareto dominated by others and is the most fair, with ``star". We defer more experimental results to \Cref{app: additional exp}.}
  \label{fig:trade-off}
\end{figure}

\textbf{Trade-off between fairness, performance, and Pareto optimality}
\Cref{fig:trade-off} visualizes the group-wise rewards as points \((R_1,R_2)\).
Points closer to the diagonal \(R_1=R_2\) indicate smaller reward disparity. We see that the unconstrained policy, $\pi^{\infty}$, exhibits the largest disparities (furthest from the diagonal).
In contrast, policy optimization with fairness using different \(\epsilon\) generally reduces reward disparity.
Across logging policies, we also observe that for at least one choice of \(\epsilon\), the unconstrained baseline does \emph{not} yield a Pareto improvement over the corresponding constrained policy from both directions, where the Pareto improvement represents the reward improvement on both groups. 
Moreover, the overall reward is often comparable to, or exceeds, the unconstrained baseline, suggesting a favorable trade-off when \(\epsilon\) is chosen appropriately.

If the user seeks to balance the trade-off between the fairness and per-group performance instead of only aiming at the target reward disparity $\epsilon$, we can slightly modify our method to achieve that.  From a high-level perspective, we can obtain a policy set with different $\epsilon$ and we select the policy that has the highest overall reward or most Pareto improvement and is most fair, i.e. with the smallest reward disparity. Practically, 
A simple, effective procedure is:
(i) sweep a small grid of \(\epsilon\) values;
(ii) compute \(\bigl(R_1^{(\epsilon)}, R_2^{(\epsilon)}\bigr)\) and \( |R_1^{(\epsilon)} - R_2^{(\epsilon)}|\);
(iii) form the Pareto frontier (non-dominated set) in the \((R_1,R_2)\)-plane; and
(iv) pick the point on this frontier that has the smallest reward disparity. 

In \Cref{fig:trade-off}, we see that the Pareto frontier when the logging policy is random in \Cref{fig:trade-off} (a) is $\pi^{0.03}$, $\pi^{0.05}$, and $\pi^{\infty}$, where $\pi^{\infty}$ is the unconstrained baseline. Thus, we can select $\pi^{0.03}$ as the policy when the logging policy is the random policy. Similarly, when the logging policy is Mixed policy, $\pi^{0.05}$, $\pi^{0.03}$ and $\pi^{\infty}$ are on the Pareto frontier and $\pi^{0.05}$ is the most fair one. We defer more results to the \Cref{app: additional exp}.

\subsection{Extension to multiple groups}
We can extend our method to multiple groups with a simple modification by only considering the most violated group pair at every Lagrangian optimization step instead of all pairs.  We formulated the optimization objective at each step as:
\begin{align}
\label{eq: most violated}
\textstyle
 \min_{\lambda \geq 0} \max_\theta V(\pi_\theta) + \lambda_{ij} \big[ \epsilon - | \mathbb{E}_{x \sim \mathcal{P}_{\mathcal{X}^i},\, a \sim \pi(\cdot \mid x)}  r(x, a) - \mathbb{E}_{x \sim \mathcal{P}_{\mathcal{X}^j},\, a \sim \pi(\cdot \mid x)}  r(x, a) |\big],
\end{align}
where $i$ and $j$ are the most violated group pair. 
For the Lagrange multiplier, we have: 
\begin{align}
\label{app_eq: dual}
    \nabla_{\lambda_{ij}} L = \epsilon - \big[ \hat{V}^{\text{DR}}_{X^j}(\pi_\theta) - \hat{V}^{\text{DR}}_{X^i}(\pi_\theta)\big], 
    \nabla_{\eta_{ij}} L = \epsilon - \big[\hat{V}^{\text{DR}}_{X^i}(\pi_\theta) - \hat{V}^{\text{DR}}_{X^j}(\pi_\theta) \big].
\end{align}
Then, we optimize the policy with:
\begin{align}
 \label{app_eq: multi-group gradient}
    \nabla_\theta L 
    &= \mathbb{E}_{x,a \sim \pi_\theta} \nabla_\theta \log \pi_\theta(a \mid x) r(x,a) \notag \\
    & = \mathbb{E}_{x,a \sim \pi_\theta, X} \nabla_\theta \log \pi_\theta(a \mid x) r(x,a)    \big[(I(x \in X^i)(1 - \lambda_{ij}+ \eta_{ij})+I(x \in X^j) (1 + \lambda_{ij} - \eta_{ij}) + I(x \notin X^{ij})  \big]. 
\end{align}

\textbf{Remark} 
\Cref{app_eq: dual} and \Cref{app_eq: multi-group gradient} are the gradient of the dual variable and the policy gradient of each step, given the most violated group pair $(i,j)$. The policy gradient will ``up-weight'' the lower reward group, ``down-weight'' the higher reward group, while keeping the same weight of all other groups. We summarize the algorithm in \Cref{algo: multigroups}.

\begin{algorithm}
\caption{Group-sensitive offline contextual bandits for multiple groups}
\label{algo: multigroups}
\begin{algorithmic}[1]
    \State \textbf{Input:} Offline dataset: $S = \{x, a, r, \pi_\beta(a|x)\}$, threshold $\epsilon$, sensitive feature. 
    \State \textbf{Output:} Learned policy $\pi_\theta$.
    \State \textbf{Initialization:} policy $\pi_\theta$, reward model $\hat{r}$, policy learning rate $\alpha$, dual variable update rate $\beta$. 
    \State Learn the reward estimator $\hat{r}$ from $S$.
    \For{t = 0, ..., T}
        \State Sample context-action pairs $D_{\pi_\theta}$.
        \State Estimate the $\hat{V}^{\text{DR}}_{X^i}(\pi_\theta)$ for every group $i$ 
        \State Select the most violated pairs $(i,j) = \argmax |\hat{V}^{\text{DR}}_{X^i}(\pi_\theta) - \hat{V}^{\text{DR}}_{X^j}(\pi_\theta) |$
        \State Update $\theta_{t+1} = 
        \theta_t + \alpha \nabla L$, with \Cref{eq:policy_gradient_with_lambda}
        \State Update the dual variable for $(i,j)$ with \Cref{app_eq: dual} and update the policy with \Cref{app_eq: multi-group gradient}
    \EndFor
    \State \textbf{Return} $\pi_{\theta_T}$
\end{algorithmic}
\end{algorithm}

We show the experiment results in \Cref{app_tab:drug_edu_multi}. We show the \emph{maximum disparity}, which is the maximum disparity between any two groups. We can see that our method can reduce the reward disparity while maintaining a similar reward.  
\begin{table}[H]
\centering
\caption{Comparison of unconstrained baseline with our method on multi-group settings. We use the Drug Consumption dataset and set Education as the sensitive feature, and three groups are separated into below high school, high school, and above high school. We set $\epsilon = 0.07$ and $\epsilon = 0.05$ and report the reward and maximum disparity between any two groups.}
\label{app_tab:drug_edu_multi}
\begin{tabular}{ll|cc|cc}
\toprule
\multirow{3}{*}{$\epsilon$} &\multirow{3}{*}{Off-policy} & \multicolumn{2}{c|}{Unconstrained} & \multicolumn{2}{c}{Ours} \\
 & & Reward & \shortstack{Maximum\\Disparity} & Reward & \shortstack{Maximum\\Disparity}  \\
\midrule
\multirow{3}{*}{0.07} & Random  & $0.515_{\pm 0.043}$ & $0.102_{\pm 0.042}$ & $0.488_{\pm 0.052}$ & $0.055_{\pm 0.021}$ \\
& Tweak-1 & $0.471_{\pm {0.033}}$ & $0.096_{\pm 0.053}$ & $0.486_{\pm 0.041}$ & $0.073_{\pm 0.028}$ \\
& Mixed& $0.492_{\pm0.038}$ & $0.094_{\pm 0.051}$ & $0.472_{\pm 0.048}$ & $0.048_{\pm 0.039}$ \\

\midrule
\multirow{3}{*}{0.05} & Random  & $0.515_{\pm 0.043}$ & $0.102_{\pm 0.042}$ & $0.484_{\pm 0.049}$ & $0.048_{\pm0.027}$ \\
& Tweak-1 & $0.471_{\pm 0.033}$ & $0.096_{\pm 0.053}$ & $0.487_{\pm 0.054}$ & $0.025_{\pm0.009}$ \\
& Mixed& $0.492_{\pm 0.038}$ & $0.094_{\pm 0.051}$ & $0.499_{\pm 0.042}$ & $0.041_{\pm 0.029}$ \\

\bottomrule
\end{tabular}

\end{table}

\section{Conclusion and Limitations}
In this paper, we tackle a new group-sensitive offline contextual bandits problem where the reward disparity between groups is controlled for policy optimization. We develop a new algorithm: off-policy group-constrained policy gradient algorithm (GC-PG) to learn a policy under constraints. We leverage a doubly robust estimator in the policy gradient-based optimization with Lagrangian relaxation. Our method addresses fairness concerns without significantly compromising the overall reward. Theoretical analysis shows that our method has a convergence rate of $O(1/T)$, while empirical results on both synthetic and real-world datasets demonstrate the effectiveness of our approach in achieving equitable outcomes across demographic groups.

\section*{Acknowledgement}
This work is partially supported by the Center for Digital Health and Artificial Intelligence (CDHAI).

\clearpage
\newpage
\appendix

\crefalias{section}{appendix}
\crefalias{subsection}{appendix}
\crefalias{subsubsection}{appendix}
\section{Theoretical Analysis}
\label{section: proof}
\begin{lemma}
\label{lemma: variance}
(Variance of the Doubly Robust Estimator, Theorem 2 in \cite{dudik2011doubly}) Let $\Delta(a, x) = \hat{r}(x, a) - r(x, a)$, let $\xi = \frac{(r(x, a) - \hat{r}(x, a))\pi_\theta(a|x)}{\pi_\beta(a|x)}$, then the variance of the doubly robust estimator is:
\begin{align*}
     &\text{Var}[\hat{V}^{\text{DR}}(\pi_\theta)] 
    =   \frac{1}{|D|} \bigg[  \mathbb{E}_{(x, a, r) \sim D}  \left[ \xi^2 \right]+ \text{Var}_x \left[ r(x,\pi_\theta(x)) + \Delta(1-\frac{\pi_b(a|x)}{\hat{\pi}_{b}(a|x)}) \right] + \mathbb{E}_{(x,a) \sim D} [ \frac{1 - \pi_\theta(a|x)}{\pi_\theta(a|x)} \cdot \Delta  ]^2 \bigg].
\end{align*}
\end{lemma}
The variance of the doubly robust estimator is derived from Theorem 2 in \cite{dudik2011doubly}, vanishes in the order of $\frac{1}{|D|}$. It consists of three terms. The first one is the reward estimation error. The second one is the variance of the reward estimation and the logging policy estimation. The third one can be viewed as the importance weight penalty. Intuitively, the high variance of the estimator will lead to high variance in the policy gradient, thereby destabilizing policy optimization.  

\begin{lemma}
(Logging policy error $1-\pi/\hat\pi$)
Let $(x_i,a_i)\sim D$ and let $\pi^*(a|x)$ be the true logging policy. Fit $\pi_\theta(a|x)$ by minimizing the cross-entropy loss $\ell(\theta;x,a):=-\log\pi_\theta(a|x)$. Assume:
	1.	$\exists \gamma\in(0,1)$ s.t. $\pi_\theta(a|x)\ge\gamma$ for all $x,a,\theta$ (so $0\le\ell\le B:=-\log\gamma$);
	2.	$\hat\theta$ is ERM: $\hat L_n(\hat\theta)\le\hat L_n(\theta^\star)$.
Under standard uniform convergence results for bounded loss classes via Rademacher complexity, with probability at least $1-\delta$,
\label{lemma: policy error}
\begin{align*}
E_{x,a}\left|1-\frac{\pi^\star(a|x)}{\pi_{\hat\theta}(a|x)}\right|\le O\left(\left(\frac{C+\log(1/\delta)}{n}\right)^{1/4}\right).
\end{align*}
\end{lemma}

\begin{proof}
    We first rewrite the cross-entropy loss with KL divergence
\begin{align}
\label{appeq: ce loss}
    L(\theta) = \mathbb{E}_x\Bigl[H\bigl(\pi^\star(\cdot| x)\bigr)+ KL\bigl(\pi^\star(\cdot| x)\|\pi_\theta(\cdot| x)\bigr)
    \Bigr].
\end{align}

Let $\theta^*$ be the parameters for $\pi^*$, we hae
\begin{align*}
  KL\bigl(\pi^\star(\cdot\mid x)\|\pi_{\theta^\star}(\cdot\mid x)\bigr) = 0,  
\end{align*}
and
\begin{align}
    \label{appeq: entropy}
    L(\theta^\star) = \mathbb{E}_x H\bigl(\pi^\star(\cdot\mid x)\bigr).
\end{align}

Subtracting \Cref{appeq: entropy} from \Cref{appeq: ce loss} yields, for any $\theta$,
$$
L(\theta)-L(\theta^\star)
= \mathbb{E}_x KL\bigl(\pi^\star(\cdot\mid x)\| \pi_\theta(\cdot\mid x)\bigr).
$$

With the ERM property we have,
$$
\hat L_n(\hat\theta) \le \hat L_n(\theta^\star).
$$
Then
$$
\begin{aligned}
L(\hat\theta) - L(\theta^\star)
&= \bigl(L(\hat\theta) - \hat L_n(\hat\theta)\bigr)
	+	\bigl(\hat L_n(\hat\theta) - \hat L_n(\theta^\star)\bigr)+	\bigl(\hat L_n(\theta^\star) - L(\theta^\star)\bigr) \\
&\le |L(\hat\theta) - \hat L_n(\hat\theta)|
	+	0
	+	|L(\theta^\star) - \hat L_n(\theta^\star)| \\
&\le \beta_n + \beta_n
= 2\beta_n \leq 2cB\sqrt{\frac{C + \log(1/\delta)}{n}},
\end{aligned}
$$
which holds with probability at least $1-\delta$ under standard Rademacher assumptions for the loss class. $C$ is a complexity measure (e.g., pseudo-dimension of $\mathcal{F}$). 
Then we have:
$$
\mathbb{E}_x KL\bigl(\pi^\star(\cdot\mid X) \| \pi_{\hat\theta}(\cdot\mid X)\bigr)
\le 2\beta_n.
$$

Now we convert the KL bound to a bound on the pointwise probability error.
For each $x$, define
$$
KL_x := KL\bigl(\pi^\star(\cdot\mid x)\|\pi_{\hat\theta}(\cdot\mid x)\bigr),
\qquad
TV_x := \frac12 \sum_{a\in\mathcal{A}}\bigl|\pi^\star(a\mid x)-\pi_{\hat\theta}(a\mid x)\bigr|.
$$
By Pinsker’s inequality,
$$
TV_x \le \sqrt{\tfrac12 KL_x}.
$$

Taking expectation over $X$ and applying Jensen’s inequality (since $\sqrt{\cdot}$ is concave),
$$
\begin{aligned}
\mathbb{E}_x[TV_x]
\le \mathbb{E}_x\Bigl[\sqrt{\tfrac12 KL_x}\Bigr]
\le\sqrt{\tfrac12\mathbb{E}_x[KL_x]} 
\le \sqrt{\tfrac12\cdot 2\beta_n}
= \sqrt{\beta_n}.
\end{aligned}
$$

Moreover, for any fixed $(x,a)$,
$$
\bigl|\pi^\star(a\mid x)-\pi_{\hat\theta}(a\mid x)\bigr|
\le \sum_{a’}\bigl|\pi^\star(a’\mid x)-\pi_{\hat\theta}(a’\mid x)\bigr|
= 2TV_x.
$$

Therefore,
$$
\begin{aligned}
\mathbb{E}_{x,a}\bigl|\pi^\star(a\mid x)-\pi_{\hat\theta}(a\mid x)\bigr|
&= \mathbb{E}_x\left[
\sum_a \pi^\star(a\mid x)
\bigl|\pi^\star(a\mid x)-\pi_{\hat\theta}(a\mid x)\bigr|
\right] \\
&\le \mathbb{E}_x\left[
\sum_a \pi^\star(a\mid x)2TV_x
\right]
= 2\mathbb{E}_x[TV_x] 
\le 2\sqrt{\beta_n}.
\end{aligned}
$$

Finally, we relate this to the relative error term.
By the assumption 1, for all $(x,a)$ we have $\pi_{\hat\theta}(a\mid x)\ge\gamma$, hence
$$
\left|1 - \frac{\pi^\star(a\mid x)}{\pi_{\hat\theta}(a\mid x)}\right|
= \left|\frac{\pi_{\hat\theta}(a\mid x)-\pi^\star(a\mid x)}{\pi_{\hat\theta}(a\mid x)}\right|
\le \frac{\bigl|\pi^\star(a\mid x)-\pi_{\hat\theta}(a\mid x)\bigr|}{\gamma}.
$$

Taking expectation over $(x,a)$ and using the previous bound,
$$
\begin{aligned}
\mathbb{E}_{x,a}
\left|
1 - \frac{\pi^\star(a\mid x)}{\pi_{\hat\theta}(a\mid x)}
\right|
&\le \frac{1}{\gamma} \cdot
\mathbb{E}_{x,a}\bigl|\pi^\star(a\mid x)-\pi_{\hat\theta}(a\mid x)\bigr| 
\le \frac{1}{\gamma}\cdot 2\sqrt{\beta_n}
= \frac{2}{\gamma}\sqrt{\beta_n}.
\end{aligned}
$$

Plugging this into the lemma yields
$$
\mathbb{E}_{x,a}
\left|
1 - \frac{\pi^\star(a\mid x)}{\pi_{\hat\theta}(a\mid x)}
\right|
= O\left(
\left(\frac{C + \log(1/\delta)}{n}\right)^{1/4}
\right),
$$
up to constants depending on $B$ and $\gamma$.
This holds with probability at least $1-\delta$, which completes the proof. $\square$ 
\end{proof}

\begin{lemma} (Reward estimation error $|r-\hat r|$) Let $(x_i,a_i,y_i) \sim D$ and let $r(x,a)$ be the true reward. Fit $f\in\mathcal F$ by MSE
$\ell(f;x,a,y):=(y-f(x,a))^2$.
Assume:
	1.	$\exists M,|y|\le M$ and $|f(x,a)|\le M$ for all $x,a,f$, hence $0\le\ell\le B_1:=(2M)^2$;
	2.	$\hat f$ is ERM: $\hat L_n(\hat f)\le\hat L_n(f^\star)$.
Under standard uniform convergence results for bounded loss classes via Rademacher complexity, with probability at least $1-\delta$,
\label{lemma: reward error}
\begin{align*}
E_{x,a}\big|\hat f(x,a)-r(x,a)\big|
\le
O\left(\left(\frac{C+\log(1/\delta)}{n}\right)^{1/4}\right).
\end{align*}
\end{lemma}

\begin{proof}
For any $f\in\mathcal F$, define the population and empirical risks
\[
L(f) := \mathbb{E}\bigl[(y-f(x,a))^2\bigr],
\qquad
\hat L_n(f) := \frac{1}{n}\sum_{i=1}^n (y_i - f(x_i,a_i))^2.
\]
By assumption, $r(x,a) = \mathbb{E}[y\mid x,a]$ is the unique minimizer of $L(f)$ over all measurable $f$, and by well-specification we have $r\in\mathcal F$, so we can take $f^\star = r$.

Since the loss is bounded, $0\le \ell\le B_1$, standard Rademacher complexity bounds for bounded loss classes (uniform convergence) imply that, with probability at least $1-\delta$,
\begin{equation}
\label{eq:uniform-conv}
\sup_{f\in\mathcal F} \bigl|L(f)-\hat L_n(f)\bigr|
\;\le\;
\varepsilon_n
:= O\!\left(\sqrt{\frac{C+\log(1/\delta)}{n}}\right),
\end{equation}
where $C$ is a complexity term depending on the (empirical) Rademacher complexity of the loss class induced by $\mathcal F$.

Because $\hat f$ is an ERM, $\hat L_n(\hat f)\le \hat L_n(f^\star)$, hence
\begin{align*}
L(\hat f) - L(f^\star)
&= \bigl(L(\hat f)-\hat L_n(\hat f)\bigr)
   + \bigl(\hat L_n(\hat f)-\hat L_n(f^\star)\bigr)
   + \bigl(\hat L_n(f^\star)-L(f^\star)\bigr) \\
&\le \bigl|L(\hat f)-\hat L_n(\hat f)\bigr|
   + 0
   + \bigl|\hat L_n(f^\star)-L(f^\star)\bigr|  \\
&\le 2\,\sup_{f\in\mathcal F} \bigl|L(f)-\hat L_n(f)\bigr|
\;\le\; 2\,\varepsilon_n.
\end{align*}
Thus
\begin{equation}
\label{eq:excess-risk}
L(\hat f) - L(f^\star)
\;\le\; 2\,\varepsilon_n
= O\!\left(\sqrt{\frac{C+\log(1/\delta)}{n}}\right).
\end{equation}

On the other hand, for squared loss with conditional mean target,
the excess risk equals the squared $L_2$ distance between $\hat f$ and $r$:
\begin{align*}
L(\hat f) - L(r)
&= \mathbb{E}\bigl[(y-\hat f(x,a))^2 - (y-r(x,a))^2\bigr] \\
&= \mathbb{E}\bigl[(\hat f(x,a)-r(x,a))^2\bigr],
\end{align*}
where we used that $r(x,a)=\mathbb{E}[y\mid x,a]$ and expand the squares.
Since $f^\star = r$, we have
\[
\mathbb{E}\bigl[(\hat f(x,a)-r(x,a))^2\bigr]
= L(\hat f) - L(f^\star)
\;\le\; 2\,\varepsilon_n.
\]

Finally, apply Jensen's inequality (or Cauchy--Schwarz) to relate $L_1$ and $L_2$ errors:
\[
\mathbb{E}_{x,a}\bigl|\hat f(x,a)-r(x,a)\bigr|
\;\le\;
\Bigl(\mathbb{E}_{x,a}\bigl[(\hat f(x,a)-r(x,a))^2\bigr]\Bigr)^{1/2}
\;\le\; (2\,\varepsilon_n)^{1/2}.
\]
Using $\varepsilon_n = O\!\left(\sqrt{\frac{C+\log(1/\delta)}{n}}\right)$, we obtain
\[
\mathbb{E}_{x,a}\bigl|\hat f(x,a)-r(x,a)\bigr|
\;\le\;
O\left(\left(\frac{C+\log(1/\delta)}{n}\right)^{1/4}\right),
\]
which proves the lemma.
\end{proof}
\setcounter{section}{4}   
\setcounter{theorem}{2}
\renewcommand{\thesection}{\arabic{section}}
\renewcommand{\thetheorem}
{\thesection.\arabic{theorem}}
Now, we provide the proof of the \Cref{lemma: bias}, \Cref{theorem: reward mse}, \Cref{theorem: approximation error}, and \Cref{theorem: gradient vanish}.
\begin{lemma} (Bias of the doubly robust estimator). The bias of the doubly robust estimator under the Rademacher complexity is bounded with $1-2\delta$ probability: 
    \begin{align*}
        \text{Bias}^2(\hat{V}^\text{Dr}(\pi_\theta)) = E_{x}[(\hat{r}(x,a) - r(x,a))(1-\frac{\pi_b(a|x)}{\hat{\pi}_b(a|x)})]^2 \leq O(\frac{\log(1/\delta)}{|D|})
    \end{align*}
where $\pi_b$ is the logging policy and $\hat{\pi}_b(a|x)$ is the estimated logging policy. 
\end{lemma}
\begin{proof}
\begin{align*}
    \text{Bias}^2(\hat{V}^\text{Dr}(\pi_\theta)) &= E_{x}[(\hat{r}(x,a) - r(x,a))(1-\frac{\pi_b(a|x)}{\hat{\pi}_b(a|x)})]^2 \\
    & = E_{x}|(\hat{r}(x,a) - r(x,a)|^2 E_{x}|1-\frac{\pi_b(a|x)}{\hat{\pi}_b(a|x)}|^2 \\
    &\leq O(\frac{\log(1/\delta)}{|D|}).
\end{align*}
The first equality is the bias of the doubly robust estimator from Theorem 1 in \cite{dudik2011doubly}. And the last inequality comes from plugging the reward estimation error bound and policy estimation error bound in \cref{lemma: policy error} and \Cref{lemma: reward error}.
\end{proof}

\begin{theorem}
\label{theorem: reward mse}
The MSE between the DR estimator and the true reward is bounded with probability $1-2\delta$ as follows:
\begin{align*}
        \text{MSE}_\text{DR} = & \mathbb{E}_{(x_i, a_i, r_i) \sim D} \bigg[ \hat{r}(x_i, \pi_{\theta}(x_i))  + \frac{\pi_{\theta}(a_i \mid x_i)}{\pi_b(a_i \mid x_i)} ( r_i - \hat{r}(x_i, a_i)) - r(x_i, \pi_\theta( x_i)) \bigg]^2 \nonumber  \\ 
        = & \text{Var}[\hat{V}^{\text{DR}}(\pi_\theta)] + \text{Bias}^2(\hat{V}^{\text{DR}}(\pi_\theta)) \\
        \leq & O(1/|D|) + O(\frac{\log(1/\delta)}{|D|})
    \end{align*}
\end{theorem}
\begin{proof}
We can directly plug in the bias and variance property of the doubly robust estimator in \Cref{lemma: bias} and \Cref{lemma: variance} to prove it. 
\end{proof}
\begin{theorem}
    (Estimation error of the gradient) 
    With probability at least $1- 2 \delta$, the generalization
bound of policy gradient estimation satisfy:  
{\small
\[
\textstyle
 \mathbb{E}_{\pi_{\theta_t}}\left[\left\| \hat{g}(\theta_t) - \nabla V^\pi \theta_t \right\|_2^2 \right] \leq G^2 (1 + 2B)^2 MSE_DR.
\]
}
\end{theorem}
\begin{proof}
    For simplicity, let $\delta_t$ defined as follow:
    \begin{align*}
    \delta_t = 
    \left\{
    \begin{array}{ll}
    1 + \lambda_t - \eta_t & \text{if } x \in X^1 \\
    1 - \lambda_t + \eta_t & \text{if } x \in X^0
    \end{array}\right.,
    \end{align*}
    where $1 + \lambda_t - \eta_t$ and $1 - \lambda_t + \eta_t$ are the dual variables from \Cref{eq:policy_gradient_with_lambda}. Then the gradient in \Cref{eq:policy_gradient_with_lambda} becomes:
    \begin{align*}
        \nabla \log \pi_{\theta_t} \delta_t \hat{r}_{\text{DR}}
    \end{align*}
    Also, following \Cref{assumption:dual bounded}, we have $\lambda_t \leq B$ and  $\eta_t \leq B$. Then $\delta_t \leq 2B + 1$. So we have:
    \begin{align*}
        \mathbb{E}_{\pi_{\theta_t}} \left[ \left\| \hat{g}(\theta_t) - \nabla V^\pi (\theta_t) \right\|_2^2 \right] &\leq \EE_{\pi_{\theta_t}} \left[ \| \nabla \log \pi_{\theta_t} \delta_t \hat{r}_{\text{DR}} - \nabla \log \pi_{\theta_t} \delta_t r \|_2^2\right] \\
        & \leq \EE_{\pi_{\theta_t}} \left[ \| \nabla \log \pi_{\theta_t}\|_2^2 \delta_t^2 (\hat{r}_{\text{DR}}- r)^2 \right] \\
        & \leq G^2(2B + 1)^2 \text{MSE}_text{DR}
    \end{align*}
\end{proof}

We prove the convergence of the optimization. We show that the policy parameter $\theta$ optimized with policy gradient converges to a stationary point such that $\| \nabla V(\pi_{\theta}) \|_2 \leq \epsilon + C$ for any $\epsilon > 0$, where $C$ is a constant determined by the offline dataset and model approximation error.

\begin{theorem}
    (Gradient vanish). Assuming the value function $V(\pi_\theta)$ is L-smooth for every $\theta$, let $T$ be the policy optimization steps, let the learning rate $\alpha = \frac{1}{4L}$. Then the learned policy $\pi_{\theta_T}$ satisfies:
    \begin{align*}
    \mathbb{E} [ \left\| \nabla V^\pi (\theta_{\hat{T}}) \right\|_2^2 ]\leq&  \frac{4}{T\eta} ( V^\pi (\theta_{T+1}) - V^\pi (\theta_1 )) + 3 G^2 (1+2B)^2  \text{MSE}_\text{DR}.
    \end{align*}
\end{theorem}
\begin{proof}
    For simplicity, we denote the $V(\theta) \overset{\triangle}{=}V(\pi_\theta)$.
    
    Given the assumption that $V(\theta)$ is $L$-smooth for every $\theta$ and the update rule $\theta_{t+1}=\theta_{t}+\alpha \hat  g(\theta_{t})$, we have
\begin{align}
    V(\theta_{t+1})&\geq V(\theta_{t})+\la\nabla V(\theta_{t}),\theta_{t+1}-\theta_{t}\ra-\frac{L}{2}\|\theta_{t+1}-\theta_{t}\|_2^2\notag\\
    &=V(\theta_{t})+\alpha\la\nabla V(\theta_{t}),\hat g(\theta_{t})\ra-\frac{L}{2}\|\theta_{t+1}-\theta_{t}\|_2^2\notag\\
    &=V(\theta_{t})+\alpha\|\nabla V(\theta_{t})\|_2^2+\alpha\la\nabla V(\theta_{t}),\hat g(\theta_{t})-\nabla V(\theta_{t})\ra-\frac{L}{2}\|\theta_{t+1}-\theta_{t}\|_2^2. \label{eq:half1}
\end{align}
Due to the fact that $\|\nabla V(\theta_{t})\|_2^2+\|\nabla g(\theta_{t})-\nabla V(\theta_{t})\|_2^2\geq 2|\la\nabla V(\theta_{t}),\hat g(\theta_{t})-\nabla V(\theta_{t})\ra|$, we have
\begin{align}
    \eqref{eq:half1}&\geq V(\theta_{t})+\alpha\|\nabla V(\theta_{t})\|_2^2-\frac{\alpha}{2}\|\nabla V(\theta_{t})\|_2^2-\frac{\alpha}{2}\|\hat g(\theta_{t})-\nabla V(\theta_{t})\|_2^2-\frac{L}{2}\alpha^2\|\hat g(\theta_{t})\|_2^2\notag\\
    &\geq V(\theta_{t})+\alpha\|\nabla V(\theta_{t})\|_2^2-\frac{\alpha}{2}\|\nabla V(\theta_{t})\|_2^2-\frac{\alpha}{2}\|\hat g(\theta_{t})-\nabla V(\theta_{t})\|_2^2\notag\\
    &\qquad-L\alpha^2\|\nabla V(\theta_{t})\|_2^2-L\alpha^2\|\hat g(\theta_{t})-\nabla V(\theta_{t})\|_2^2\label{eq:cauchy}\\
    &=V(\theta_{t})+(\alpha/2-L\alpha^2)\|\nabla V(\theta_{t})\|_2^2-(\alpha/2+L\alpha^2)\|\hat g(\theta_{t})-\nabla V(\theta_{t})\|_2^2\notag,
\end{align}
where \eqref{eq:cauchy} holds due to mean inequality  $\|\nabla V(\theta_{t})\|_2^2+\|\hat g(\theta_{t})-\nabla V(\theta_{t})\|_2^2\geq \frac 12 \|\hat g(\theta_{t})\|^2$.

Rearranging the above inequality yields
\begin{align*}
    (\alpha/2-L\alpha^2)\|\nabla V(\theta_{t})\|_2^2\leq V(\theta_{t+1})-V(\theta_{t})+(\alpha/2+L\alpha^2)\|\hat g(\theta_{t})-\nabla V(\theta_{t})\|_2^2.
\end{align*}
Now we take the expectation over $\pi_{\theta_t}$ on both sides of the above inequality and then sum it up over $t=1,2,\cdots,T$. We have
\begin{align*}
    \sum_{t=1}^{T}\EE_{\pi_{\theta_t}}\big[\|\nabla V(\theta_{t})\|_2^2\big]\leq\frac{V(\theta_{T+1})-V(\theta_{1})}{\alpha/2-L\alpha^2}+\frac{\alpha/2+L\alpha^2}{\alpha/2-L\alpha^2}\sum_{t=1}^{T}\EE_{\pi_{\theta_t}}\big[\| \hat g(\theta_{t})-\nabla V(\theta_{t})\|_2^2\big].
\end{align*}
Setting step size as $\alpha =  1/(4L)$ in the above inequality yields
\begin{align}
\label{eq:stepsize}
    \sum_{t=1}^{T}\EE_{\pi_{\theta_t}}\big[\|\nabla V(\theta_{t})\|_2^2\big]\leq \frac{V(\theta_{T+1})-V(\theta_{1})}{\alpha/4}+3\sum_{t=1}^{T}\EE_{\pi_{\theta_t}}\big[\|\hat g(\theta_{t})-\nabla V(\theta_{t})\|_2^2\big].
\end{align}
We define  $\theta_{\hat T}$ such that the index $\hat T$ is randomly picked from $\{1,\ldots,T\}$. By Jensen's inequality,  we have
\begin{align}
    \EE[\|\nabla V(\theta_{\hat T})\|_2^2]&\leq \frac{1}{T}\sum_{t=1}^{T}\EE_{\pi_{\theta_t}}\big[\|\nabla V(\theta_{t})\|_2^2\big]\label{eq:Jensen}.
\end{align}
Combining \Cref{eq:Jensen} with \Cref{eq:stepsize} and subsequently with Theorem \Cref{theorem: approximation error}, we can have
\begin{align*}
    \eqref{eq:Jensen}&\leq \frac{V(\theta_{T+1})-V(\theta_{1})}{T\alpha/4}+\frac{3}{T}\sum_{t=1}^{T}\EE_{\pi_{\theta_t}}\big[\|\hat g(\theta_{t})-\nabla V(\theta_{t})\|_2^2\big]\\
     &\leq \frac{V^{\pi_{\theta_{T+1}}}-V^{\pi_{\theta_1}}}{T\alpha/4}+3 G^2 (1+2B)^2  \text{MSE}_\text{DR}.
\end{align*}
 This completes the proof.
\end{proof}

\setcounter{section}{1}   
\renewcommand{\thesection}{\Alph{section}}
\renewcommand{\thetheorem}{\thesection.\arabic{theorem}}

\section{Additional Experiments}
\label{app: additional exp}

In this section, we present additional experimental results, including the experimental results on different $\epsilon$, and extending our method to multi-group settings. 

\subsection{Existence of reward disparity.}
In \Cref{tab_app:before_after_diff_arrow app}, we show additional experimental results with different logging policies to demonstrate that the reward disparity increases with policy optimization and unconstrained policies favor one group over another, showing performance gaps. 

\begin{table*}[ht]
\setlength{\tabcolsep}{4pt} 
\centering
\caption{Comparison of original reward v.s. after policy optimization without constraint. This demonstrates that the reward disparity increases with policy optimization without fairness. }
\label{tab_app:before_after_diff_arrow app}
\begin{tabular}{c|ccccccc}
\toprule
\multirow{2}{*}{\textbf{Off-policy}}&\multirow{2}{*}{\textbf{Data}}&\multirow{2}{*}{\textbf{Feature}} &\multicolumn{2}{c}{\textbf{Offline}} & \multicolumn{2}{c}{\textbf{Optimized}} & \multirow{2}{*}{\textbf{Disparity Increase}} \\
\cmidrule(lr){4-5} \cmidrule(lr){6-7}
&&&\textbf{$X^0$} & \textbf{$X^1$} & \textbf{$X^0$} & \textbf{$X^1$} & \\
\midrule
\multirow{5}{*}{Random} & Income & Gender & $0.498$ & $0.501$ & $0.902$ & $0.737$ & $0.003 \rightarrow 0.165$ \\
& Income & Race & $0.500$ & $0.502$ & $0.768$ & $0.852$ & $0.002 \rightarrow 0.084$ \\
& Drug & Gender & $0.252$ & $0.231$ & $0.515$ & $0.426$ & $0.021 \rightarrow 0.089$ \\
& Drug & Edu & $0.250$ & $0.256$ & $0.475$ & $0.561$ & $0.006 \rightarrow 0.086$ \\
& Student & Edu & $0.239$ & $0.218$ & $0.694$ & $0.556$ & $0.021 \rightarrow 0.138$ \\
\midrule
\multirow{5}{*}{Tweak-1} & Income & Gender & $0.385$ & $0.271$ & $0.901$ & $0.729$ & $0.114 \rightarrow 0.172$ \\
& Income & Race & $0.289$ & $0.354$ & $0.765$ & $0.857$ & $0.065 \rightarrow 0.092$ \\
& Drug & Gender & $0.310$ & $0.221$ & $0.519$ & $0.425$ & $0.089 \rightarrow 0.094$ \\
& Drug & Edu & $0.176$ & $0.307$ & $0.461$ & $0.533$ & $0.131 \rightarrow 0.072$ \\
& Student & Edu & $0.412$ & $0.376$ & $0.753$ & $0.595$ & $0.036 \rightarrow 0.158$ \\
\bottomrule
\end{tabular}
\end{table*}

\subsection{Additional experiment results}
In \Cref{tab: additional exp}, we show the additional experiment results on different $\epsilon$ on different datasets. Our method receives a higher reward compared with RobinHood, while receiving a similar reward disparity across different $\epsilon$. 

\begin{table*}[ht]
\centering
\setlength{\tabcolsep}{6pt}
\caption{Additional experiments on different $\epsilon$. Comparison of the \textbf{Unconstrained baseline} denoted as $\pi^\infty$ vs. \textbf{RobinHood} and \textbf{Ours}. Our method receives a lower reward disparity and also a similar overall reward compared with the unconstrained baseline. Our method also outperforms the RobinHood baseline, yielding higher rewards and a similar reward disparity around the $\epsilon$.}
\label{tab: additional exp}
\resizebox{\textwidth}{!}{\begin{tabular}{l l l | cc | cc | cc}
\toprule
\multicolumn{3}{c|}{} & \multicolumn{2}{c|}{\textbf{Unconstrained}} & \multicolumn{2}{c|}{\textbf{RobinHood}} & \multicolumn{2}{c}{\textbf{Ours}} \\
Dataset & Group & Off-policy & $\text{Reward}$ & $\text{Disparity}$ & $\text{Reward}$ & $\text{Disparity}$ & $\text{Reward}$ & $\text{Disparity}$ \\
\midrule
\multirow{6}{*}{Drug} & \multirow{3}{*}{\shortstack{Edu\\$\epsilon=0.05$}} 
  & Random  & $0.526_{\pm 0.034}$ & $0.086_{\pm 0.043}$ & $0.501_{\pm 0.012}$ & $0.035_{\pm 0.036}$ &           $0.533_{\pm 0.042}$  & $0.036_{\pm 0.019}$ \\
& & Tweak-1 & $0.501_{\pm 0.021}$ & $0.072_{\pm 0.031}$ & $0.449_{\pm 0.009}$ & $0.062_{\pm 0.031}$ &           $0.489_{\pm 0.055}$  & $0.052_{\pm 0.038}$ \\
& & Mixed   & $0.516_{\pm 0.051}$ & $0.073_{\pm 0.032}$ & $0.483_{\pm 0.013}$ & $0.052_{\pm 0.029}$ &           $0.514_{\pm 0.054}$  &  $0.046_{\pm 0.022}$ \\
\cmidrule(lr){2-9}
& \multirow{3}{*}{\shortstack{Gender\\$\epsilon=0.05$}}
  & Random  & $0.470_{\pm 0.023}$ & $0.081_{\pm 0.015}$ & $ 0.509_{\pm 0.028}$  & $0.060_{\pm 0.017}$          & $0.479_{\pm 0.029}$  & $0.044_{\pm 0.013}$ \\
& & Tweak-1 & $0.463_{\pm 0.050}$ & $0.094_{\pm 0.039}$ & $ 0.490_{\pm 0.039}$  & $0.048_{\pm 0.016}$          & $0.471_{\pm 0.055}$  & $0.064_{\pm 0.032}$ \\
& & Mixed   & $0.476_{\pm 0.039}$ & $0.116_{\pm 0.041}$ & $ 0.521_{\pm 0.055}$  & $0.068_{\pm 0.022}$          & $0.488_{\pm 0.072}$  & $0.046_{\pm 0.029}$ \\
\midrule
\multirow{9}{*}{Adult} & \multirow{3}{*}{\shortstack{Gender\\$\epsilon=0.1$}}
  & Random  & $0.792_{\pm 0.014}$ & $0.165_{\pm 0.015}$ &      $ 0.773_{\pm 0.042}$  & $0.101_{\pm 0.015}$      & $0.812_{\pm 0.043}$  & $0.089_{\pm 0.023}$ \\
& & Tweak-1 & $0.787_{\pm 0.039}$ & $0.172_{\pm 0.018}$ &      $ 0.768_{\pm 0.022}$  & $0.062_{\pm 0.009}$      & $0.802_{\pm 0.098}$  & $0.105_{\pm 0.021}$ \\
& & Mixed   & $0.797_{\pm 0.067}$ & $0.154_{\pm 0.019}$ &      $ 0.764_{\pm 0.017}$  & $0.077_{\pm 0.014}$      & $0.787_{\pm 0.082}$  & $0.01_{\pm 0.031}$  \\
\cmidrule(lr){2-9}
&  \multirow{3}{*}{\shortstack{Gender\\$\epsilon=0.05$}}
  & Random  & $0.792_{\pm 0.014}$ & $0.165_{\pm 0.015}$ &      $ 0.765_{\pm 0.038}$  & $0.086_{\pm 0.007}$      & $0.777_{\pm 0.087}$  & $0.046_{\pm 0.026}$ \\
& & Tweak-1 & $0.787_{\pm 0.039}$ & $0.172_{\pm 0.018}$ &      $ 0.754_{\pm 0.035}$  & $0.048_{\pm 0.023}$      & $0.709_{\pm 0.021}$  & $0.039_{\pm 0.031}$ \\
& & Mixed   & $0.797_{\pm 0.067}$ & $0.154_{\pm 0.019}$ &      $ 0.758_{\pm 0.024}$  & $0.049_{\pm 0.015}$      & $0.751_{\pm 0.070}$  & $0.053_{\pm 0.011}$ \\
\cmidrule(lr){2-9}
& \multirow{3}{*}{\shortstack{Race\\$\epsilon=0.03$}}
  & Random  & $0.805_{\pm 0.053}$ & $0.084_{\pm 0.005}$ &      $0.768_{\pm 0.019}$ & $0.035_{\pm 0.004}$      & $0.813_{\pm 0.049}$ & $0.054_{\pm 0.039}$ \\
& & Tweak-1 & $0.797_{\pm 0.047}$ & $0.092_{\pm 0.014}$ &      $0.777_{\pm 0.031}$ & $0.051_{\pm 0.014}$      & $0.799_{\pm 0.051}$ & $0.049_{\pm 0.021}$ \\
& & Mixed   & $0.805_{\pm 0.051}$ & $0.074_{\pm 0.027}$ &      $0.784_{\pm 0.032}$ & $0.068_{\pm 0.019}$      & $0.797_{\pm 0.022}$ & $0.066_{\pm 0.013}$ \\
\midrule
\multirow{3}{*}{Student} & \multirow{3}{*}{\shortstack{Edu\\$\epsilon=0.1$}}
  & Random  & $0.611_{\pm 0.017}$ & $0.138_{\pm 0.041}$ &      $ 0.515_{\pm 0.014}$  & $0.029_{\pm 0.023}$      & $0.607_{\pm 0.043}$ & $0.067_{\pm 0.032}$ \\
& & Tweak-1 & $0.672_{\pm 0.049}$ & $0.158_{\pm 0.039}$ &      $ 0.469_{\pm 0.011}$  & $0.028_{\pm 0.018}$      & $0.640_{\pm 0.041}$ & $0.083_{\pm 0.011}$  \\
& & Mixed   & $0.649_{\pm 0.042}$ & $0.132_{\pm 0.031}$ &      $ 0.502_{\pm 0.028}$  & $0.074_{\pm 0.042}$       & $0.626_{\pm 0.025}$ & $0.078_{\pm 0.009}$ \\
\bottomrule
\end{tabular}}
\end{table*}

\subsection{Trade-off between performance and fairness} 
\begin{figure*}[ht]
  \centering
  \setlength{\tabcolsep}{3pt}
  \begin{tabular}{@{}c c c@{}} %
    \includegraphics[width=.28\linewidth]{figure/group_0_scatter.pdf} &
    \includegraphics[width=.28\linewidth]{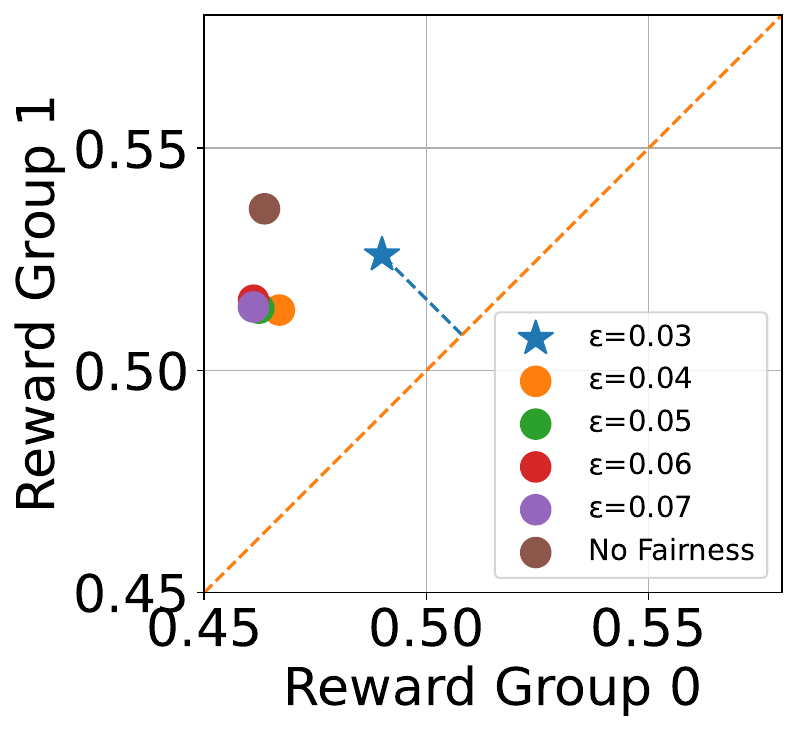} &
    \includegraphics[width=.28\linewidth]{figure/group_2_scatter.pdf} \\
    \small Random Policy & \small Tweak-1 Policy & \small Mixed Policy\\
    &\textbf{Education} & \\
    \includegraphics[width=.28\linewidth]{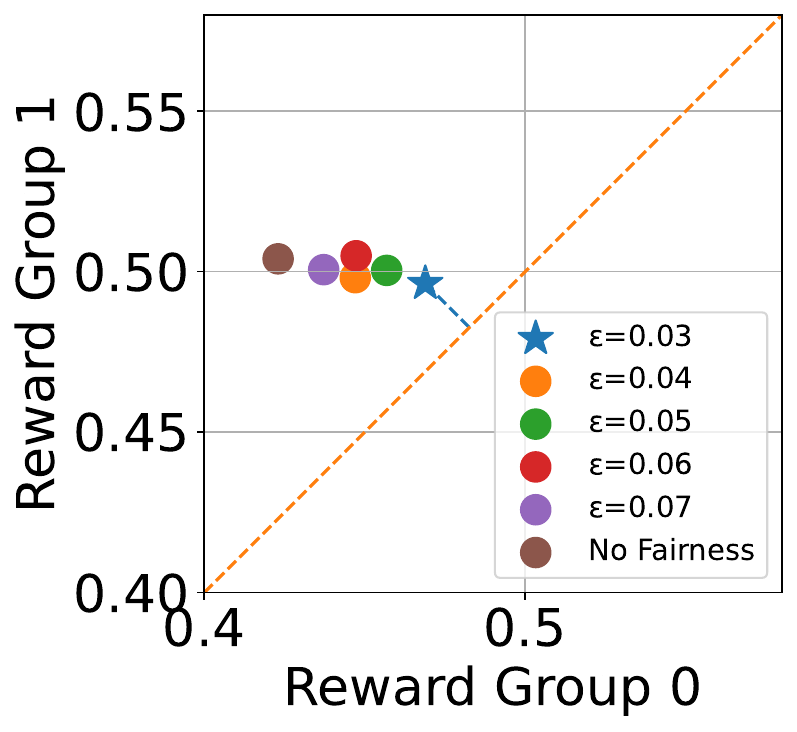} &
    \includegraphics[width=.28\linewidth]{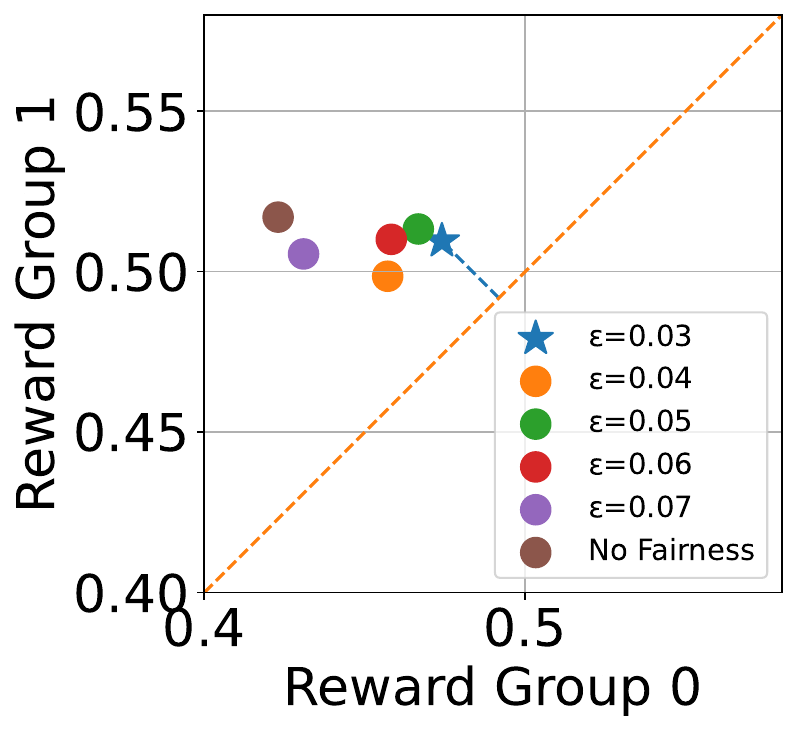} &
    \includegraphics[width=.28\linewidth]{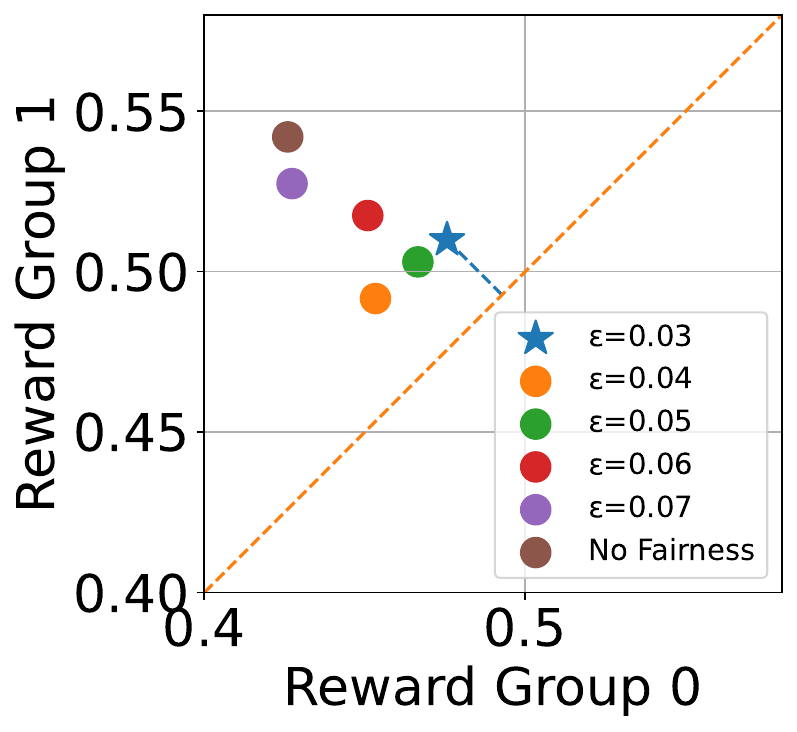} \\
    \small Random Policy & \small Tweak-1 Policy & \small Mixed Policy\\
    &\textbf{Gender} & \\
  \end{tabular}
  \caption{Scatter plot of reward $(R_1, R_2)$ on two groups using three different logging policies with different $\epsilon$ on Drug dataset. First row: education as the sensitive feature; second row: gender as the sensitive feature. Our method can reduce the reward disparity, but also has a trade-off between fairness and per-group reward. We mark the best policy with ``star" (Not Pareto dominated by others and most fair). }
  \label{app_fig:trade-off app}
\end{figure*}

\textbf{Pareto Optimality}
In this section, we discuss the Pareto optimality of our algorithm. We first define the $\varepsilon$-Fair Pareto Optimality and Global Pareto Optimality. And then we discuss how our algorithm can be used with minimal sacrifice of performance while improving the group fairness through a slight modification of the algorithm.   

\begin{definition}[$\varepsilon$-Fair Pareto Optimality]\label{def:fair-pareto}
Let two groups’ expected rewards under policy $\pi$ be $(R_1^{\pi}, R_2^{\pi})$.
Define the $\varepsilon$-fair policy set
\[
  \Pi_{\mathrm{fair}} \;=\; \bigl\{\pi \;:\; \lvert R_1^{\pi}-R_2^{\pi}\rvert \le \varepsilon \bigr\}.
\]
A policy $\pi \in \Pi_{\mathrm{fair}}$ is \emph{Pareto optimal (within $\Pi_{\mathrm{fair}}$)} if there is no
$\pi' \in \Pi_{\mathrm{fair}}$ such that
\[
  R_1^{\pi'} \ge R_1^{\pi} \quad \text{and} \quad R_2^{\pi'} \ge R_2^{\pi},
\]
with at least one inequality strict.
\end{definition}

\begin{definition}[Global Pareto Optimality]\label{def:pareto}
Let two groups’ expected rewards under policy $\pi$ be $(R_1^{\pi}, R_2^{\pi})$. A policy is \emph{Pareto optimal} if there is no
$\pi' \in \Pi$, where $\Pi$ is policy set, such that
\[
  R_1^{\pi'} \ge R_1^{\pi} \quad \text{and} \quad R_2^{\pi'} \ge R_2^{\pi},
\]
with at least one inequality strict.
\end{definition}

\textbf{Remark} The \Cref{def:fair-pareto} and \Cref{def:pareto} define the Pareto optimality with or without the group-wise fairness. We show that the global optimality of our algorithm is $\epsilon$-Fair Pareto Optimal. However, our algorithm might not be globally Pareto Optimal in some settings, especially with very small $\epsilon$, as the $\epsilon$ fairness defined by the user might be achieved by sacrificing one group's reward. Alternatively, one can still seek to achieve the Global Pareto Optimality with a slight modification of the algorithm. From the high-level perspective, we can obtain a Global Pareto Optimal set $\Pi_{\text{GPO}}$, where each policy in the $\Pi_{\text{GPO}}$ is Global Pareto Optimal. Then, we select the policy in $\Pi_{\text{GPO}}$ that is most fair, i.e., with the smallest reward disparity. Practically, we can achieve this by running the algorithm with different $\epsilon$ and obtaining multiple policies. Then we select the global Pareto optimal one with the smallest reward disparity. We summarize the algorithm in \Cref{algo: tradeoff}.

\begin{algorithm}
\caption{Balancing the per-group reward and fairness.}
\label{algo: tradeoff}
\begin{algorithmic}[1]
    \State \textbf{Input:} Offline dataset: $S = \{x, a, r, \pi_\beta(a|x)\}$, threshold $\epsilon$, sensitive feature. 
    \State \textbf{Initialization:} A set of $\mathcal{E}$.  
    \State Learn the reward estimator $\hat{r}$ from $S$.
    \For{$\epsilon \text{ in } \mathcal{E}$}
    \State $\theta_{\epsilon} \leftarrow $ Call Algorithm 1
    \State Evaluate per-group reward $(V_{X^0}(\pi_\epsilon),V_{X^1}(\pi_\epsilon)) $
    \EndFor
    \State \textbf{Return}  The Pareto optimal policy with the smallest reward disparity  $|V_{X^0}(\pi_\epsilon)  -  V_{X^1}(\pi_\epsilon)|$
\end{algorithmic}
\end{algorithm}

We present additional experimental results showing the trade-off between performance and the fairness constraint under different values of $\epsilon$, using \Cref{algo: tradeoff} in \Cref{app_fig:trade-off app}. As shown, there exists a trade-off between the per-group rewards and fairness: improving the reward for one group typically leads to a decrease in another. Among all policies, we mark the one lying on the Pareto frontier with the smallest reward disparity as the best policy.

\section{Experiment details}

\textbf{Reward estimation}. We use XGBoost for the reward estimation $\hat{r} = f(x,a)$. We implement it with scikit-learn, and set the max depth to be 5, L2 regularization $\gamma = 5$, subsample $=0.8$, n\_estimators $= 100$, and $\lambda = 0.1$. Others remain the default in the package.

\textbf{Policy network}. The hidden layer of the policy network is $(256, 256)$ and we use RELU as the activation function. 

\textbf{Other hyperparameters}. The upper bound of the dual variable is set to be $B=0.5$, the learning rate of the dual variable is tuned among $\{0.05,0.1,0.5,1\}$, the learning rate of the policy network is tuned among $\{5e-4,1e-3,5e-3,1e-2\}$.

Our codes are available at 
\url{https://github.com/guoyihonggyh/group-sensitive-fair-bandit.git}

\newpage
\bibliography{ref}
\bibliographystyle{ims}
\end{document}